\documentclass[sigconf]{acmart}
\AtBeginDocument{%
  \providecommand\BibTeX{{%
    \normalfont B\kern-0.5em{\scshape i\kern-0.25em b}\kern-0.8em\TeX}}}

\usepackage{enumitem}
\usepackage{multirow}
\usepackage{caption}
\usepackage{graphicx}
\usepackage{subcaption}
\usepackage{algorithm}
\usepackage{algpseudocode}
\usepackage{hyperref}
\usepackage{amsfonts}
\usepackage{url}
\usepackage{makecell}
\usepackage{color}
\usepackage{xspace}
\usepackage{amsmath}
\usepackage{adjustbox}
\usepackage{booktabs}
\usepackage{soul}
\usepackage{float}
\usepackage{amsthm}
\usepackage{scalerel}
\usepackage{thmtools, thm-restate}

\algnewcommand{\Initialize}[1]{%
  \State \textbf{Initialize:}
  \Statex \hspace*{\algorithmicindent}\parbox[t]{.8\linewidth}{\raggedright #1}
}
\DeclareMathOperator*{\concat}{\scalerel*{\Vert}{\sum}}

\definecolor{firstcolor}{HTML}{009E73}
\definecolor{secondcolor}{HTML}{0072B2}
\definecolor{thirdcolor}{HTML}{D55E00}
\definecolor{truepair}{HTML}{00CC00}
\definecolor{falsepair}{HTML}{FF8000}

\newcommand\colorfirst[1]{\textcolor{firstcolor}{\textbf{#1}}}
\newcommand\colorsecond[1]{\textcolor{secondcolor}{\textbf{#1}}}\newcommand\colorthird[1]{\textcolor{thirdcolor}{\textbf{#1}}}

\copyrightyear{2024}
\acmYear{2024}
\setcopyright{rightsretained}
\acmConference[KDD '24]{Proceedings of the 30th ACM SIGKDD Conference on Knowledge Discovery and Data Mining}{August 25--29, 2024}{Barcelona, Spain}
\acmBooktitle{Proceedings of the 30th ACM SIGKDD Conference on Knowledge Discovery and Data Mining (KDD '24), August 25--29, 2024, Barcelona, Spain}\acmDOI{10.1145/3637528.3672025}
\acmISBN{979-8-4007-0490-1/24/08}

\author{Harry Shomer}
\email{shomerha@msu.edu}
\affiliation{%
  \institution{Michigan State University}
  \city{East Lansing}
  \country{USA}
}

\author{Yao Ma}
\email{may13@rpi.edu}
\affiliation{%
  \institution{Rensselaer Polytechnic Institute}
  \city{Troy}
  \country{USA}
}

\author{Haitao Mao}
\email{haitaoma@msu.edu}
\affiliation{%
  \institution{Michigan State University}
  \city{East Lansing}
  \country{USA}
}

\author{Juanhui Li}
\email{lijuanh1@msu.edu}
\affiliation{%
  \institution{Michigan State University}
  \city{East Lansing}
  \country{USA}
}

\author{Bo Wu}
\email{bwu@mines.edu}
\affiliation{%
  \institution{Colorado School of Mines}
  \city{Golden}
  \country{USA}
}

\author{Jiliang Tang}
\email{tangjili@msu.edu}
\affiliation{%
  \institution{Michigan State University}
  \city{East Lansing}
  \country{USA}
}

\begin{document}

\title{LPFormer: An Adaptive Graph Transformer  for Link Prediction}
\renewcommand{\shorttitle}{LPFormer: An Adaptive Graph Transformer for LP}

\begin{abstract}
Link prediction is a common task on graph-structured data that has seen applications in a variety of domains. Classically, hand-crafted heuristics were used for this task. Heuristic measures are chosen such that they correlate well with the underlying factors related to link formation. In recent years, a new class of methods has emerged that combines the advantages of message-passing neural networks (MPNN) and heuristics methods. These methods perform predictions by using the output of an MPNN in conjunction with a ``pairwise encoding'' that captures the relationship between nodes in the candidate link. They have been shown to achieve strong performance on numerous datasets. However, current pairwise encodings often contain a strong  inductive bias, using the same underlying factors to classify all links. This limits the ability of existing methods to learn how to properly classify a variety of different links that may form from different factors. To address this limitation, we propose a new method, {\bf LPFormer}, which attempts to adaptively learn the pairwise encodings for each link. LPFormer models the link factors via an attention module that learns the pairwise encoding that exists between nodes by modeling multiple factors integral to link prediction. Extensive experiments demonstrate that LPFormer can achieve SOTA performance on numerous datasets while maintaining efficiency. The code is available at The code is available at \href{https://github.com/HarryShomer/LPFormer}{\color{blue}{https://github.com/HarryShomer/LPFormer}}.
\end{abstract}



\maketitle

\section{Introduction} \label{sec:intro}

\begin{figure}[t]
\centering
      \includegraphics[width=0.6\linewidth]{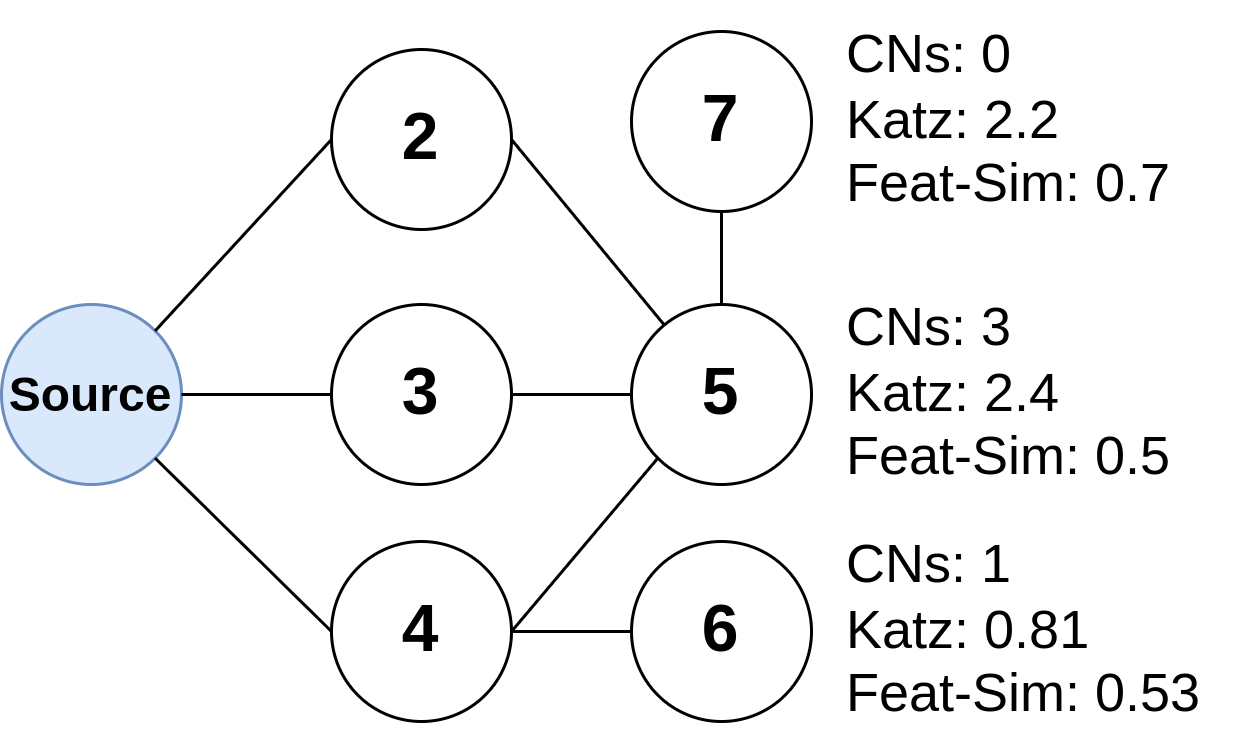}
        \caption{Example of multiple heuristic scores for the candidate links (source, 5), (source, 6), and (source, 7). Each heuristic corresponds to a different LP factor -- local (CNs), global (Katz), and feature proximity (Feat-Sim).}
\label{fig:example}

\end{figure}

Link prediction (LP) attempts to predict unseen edges in a graph. It has been adopted in many applications including recommender systems~\cite{huang2005link}, social networks~\cite{daud2020applications}, and drug discovery~\cite{abbas2021application}. Traditionally, hand-crafted heuristics were used to identify new links in the graph~\cite{newman2001clustering, zhou2009predicting, adamic2003friends}.
Heuristics are often chosen based on factors that typically correlate well with the formation of new links.  For example, a popular heuristic is common neighbors (CNs), which assume that the links are more likely to exist between node pairs with more shared neighbors. It has been found that these factors, which we refer to as ``LP Factors'', often stem from the local and global structural information and feature proximity~\cite{mao2023revisiting}. 
We give an example in Figure~\ref{fig:example} that demonstrates different heuristic scores for multiple candidate links. Each heuristic score corresponds to one of the LP factors: CNs for local information, Katz for global, and Feat-Sim for feature proximity. We can observe that the pair (source, 5) has the highest CN and Katz score of the candidate links, indicating an abundance of local and global structural information between the pair. On the other hand, the feature similarity for (source, 5) is the lowest among the candidate links. This indicates that {\it different LP factors and heuristics have distinct assumptions about why links are formed}. 

More recently, message passing neural networks (MPNNs)~\cite{gilmer2017neural}, which are able to learn effective node representations via message passing, have been widely adopted for LP tasks. They predict the existence of a link by combining the node representations of both nodes in the link. However, such a node-centric view is unable to incorporate the pairwise information between the nodes in the link. Because of this, conventional MPNNs have been demonstrated to be poor link predictors due to their limited capability to learn effective and expressive link representations~\cite{labeling_trick, srinivasan2019equivalence}. To address this issue, recent efforts~\cite{seal, nbfnet} have attempted to move beyond the node-centric view of traditional MPNNs by equipping them with pairwise information specific to the link being predicted (i.e. the ``target link'')~\cite{seal, nbfnet}. This is done by customizing the message passing process to each target link. 
However, a concern with this approach is that it can be prohibitively expensive~\cite{chamberlain2022graph}, as message passing needs to be run for each individual target link. This is as opposed to traditional MPNNs which only run message passing once for all target links.   

To overcome these inefficiencies, recent methods~\cite{yun2021neo, chamberlain2022graph, ncn} have instead explored ways to inject pairwise information into the model, without individualizing the message passing to each target link. This is done by decoupling the message passing and link-specific pairwise information. By doing so, the message passing only needs to be done once for all target links. To include the pairwise information, these methods, which we refer to as ``Decoupled Pairwise MPNNs'' (DP-MPNNs), instead learn a ``pairwise encoding'' to encode the pairwise relationship of the target link. The choice of pairwise encoding is often based on heuristics that correspond to common LP factors (e.g., common neighbors). DP-MPNNs have gained attention as they can achieve promising performance while being much more efficient than methods that customize the message passing mechanism.

However, DP-MPNNs are often limited in the choice of pairwise encoding, using a one-size-fits-all solution for all target links. This has two limitations. {\bf (1)} The pairwise encoding may fail to consider some integral LP factors. For example, NCNC~\cite{ncn} only considers the 1-hop neighborhood when computing the pairwise encoding, thereby ignoring the global structural information. {\it This suggests the need for a pairwise encoding that considers multiple types of LP factors}. {\bf (2)} The pairwise encoding uses the same LP factors for all target links. This assumes that all target links need the same factors. However, it may not necessarily be true. Recently,~\citet{mao2023revisiting} have shown that different LP factors are necessary to classify different target links. It is evident that even for the same dataset, multiple LP factors are needed to properly predict all target links. This further applies to different datasets, where certain factors are more prominent than others. As such, it faces tremendous challenges when considering multiple types of LP factors. While one factor may effectively model some target links, it will fail for other target links where those patterns aren't present. 
{\it It is therefore desired to consider different LP factors for different target links}.

These observations motivate us to ask -- {\it can we design an efficient method that can adaptively determine which LP factors to incorporate for each individual target link?} Essentially, it requires a pairwise encoding that (a) models multiple LP factors, (b) can be tailored to fit each individual target link, and (c) is efficient to calculate. By doing so, we can flexibly adapt the pairwise information based on the existing needs of each target link.   
To achieve this, we propose {\bf LPFormer} -- {\bf L}ink {\bf P}rediction Trans{\bf Former}. LPFormer is a type of graph Transformer~\cite{muller2023attending} designed specifically for link prediction. Given a target link $(a, b)$, LPFormer models the pairwise encoding via an attention module that learns how $a$ and $b$ relate in the context of various LP factors. This allows for a more customizable set of pairwise encodings that are specific to each target link. Extensive experiments validate that LPFormer can achieve SOTA on a variety of benchmark datasets. We further demonstrate that LPFormer is better at modeling several types of LP factors, highlighting its adaptability, while also maintaining efficiency on denser graphs.

\section{Background} \label{sec:background}

\subsection{Related Work} \label{sec:related_work}

Link prediction (LP) aims to model how links are formed in a graph. The process by which links are formed, i.e., link formation, is often governed by a set of underlying factors~\cite{barabasi2002evolution, liben2003link}. We refer to these as ``LP factors''. Two categories of methods are used for modeling these factors -- heuristics and MPNNs. We describe each class of methods. We further include a discussion on existing graph transformers.

{\bf Heuristics for Link Prediction}. Heuristics methods~\cite{newman2001clustering, zhou2009predicting} attempt to explicitly model the LP factors via hand-crafted measures. Recently, \citet{mao2023revisiting} have shown that there are three main factors that correlate with the existence of a link: (1) local structural information, (2) global structural information, and (3) feature proximity. {\bf Local structural information} only considers the immediate neighborhood of the target link. Representative methods include Common Neighbors (CN)~\cite{newman2001clustering}, Adamic Adar  
 (AA)~\cite{adamic2003friends}, and Resource Allocation  (RA)~\cite{zhou2009predicting}. They are 
 predicated on the assumption that nodes that share a greater number neighbors exhibit a higher probability of forming connections. {\bf Global structural information} further considers the global structure of the graph. Such methods include Katz~\cite{katz1953new} and Personalized PageRank~(PPR)~\cite{pagerank}. 
These methods posit that nodes interconnected by a higher number of paths are deemed to have larger similarity and, therefore, are more likely to form connections. 
Lastly, {\bf feature proximity} assumes nodes with more similar features connect~\cite{murase2019structural}. Previous work~\cite{nickel2014reducing, zhao2017leveraging} have shown that leveraging the node features are helpful in predicting links.
Lastly, we note that \citet{mao2023revisiting} has recently shown that {\it to properly predict a wide variety of links, it's integral to incorporate all three of these factors}.

{\bf MPNNs for Link Prediction}.  Message Passing Neural Networks (MPNNs)~\cite{gilmer2017neural} aim to learn node representations via the message passing mechanism. Traditional MPNNs have been used for LP including GCN~\cite{kipf2016semi}, SAGE~\cite{hamilton2017inductive}, and GAE~\cite{kipf2016variational}. However, they have been shown to be suboptimal for LP as they aren't expressive enough to capture important pairwise patterns~\cite{zhang2021labeling, srinivasan2019equivalence}. SEAL~\cite{seal} and NBFNet~\cite{nbfnet} try to address this by customizing the message passing process to each target link. This allows for the message passing to learn pairwise information specific to the target link. However, these methods have been shown to be unduly expensive as they require a separate round of message passing for each target link. As such, recent methods have been proposed to instead decouple the message passing and pairwise information~\cite{yun2021neo,chamberlain2022graph,ncn}, reducing the time needed to do message passing. Such methods include NCN/NCNC~\cite{ncn} which exploit the common neighbor information and BUDDY~\cite{chamberlain2022graph} and Neo-GNN~\cite{yun2021neo} which consider the global structural information. 

{\bf Graph Transformers}. Recent work has attempted to extend the original Transformer~\cite{vaswani2017attention}
architecture to graph-structured data. Graphormer~\cite{graphormer} learns node representations by attending all nodes to each other. To properly model the structural information, they propose to use multiple types of structural encodings (i.e., structural, centrality, and edge). SAN~\cite{kreuzer2021rethinking} further considers the use of the Laplacian positional encodings (LPEs) to enhance the learnt structural information. Alternatively, TokenGT~\cite{kim2022pure} considers all nodes and edges as tokens in the sequence when performing attention. Due to the large complexity of these models, they are unable to scale to larger graphs. To address this, several graph transformers~\cite{chen2022nagphormer, wu2022nodeformer} have been proposed for node classification that attempt to efficiently attend to the graph. However, while some work~\cite{chen2021hitter, pahuja2023retrieve} have formulated transformers for knowledge graph completion, to our knowledge, {\bf there are no graph transformers designed specifically for LP on uni-relational graphs}.

\subsection{Preliminaries}

We denote a graph as $\mathcal{G} = \{\mathcal{V}, \mathcal{E}\}$, where $\mathcal{V}$ and $\mathcal{E}$ are the sets of nodes and edges in $\mathcal{G}$, respectively. The adjacency matrix is represented as $A \in \mathbb{R}^{\lvert V \rvert \times \lvert V \rvert}$. The $d$-dimensional node features are represented by the matrix $X \in \mathbb{R}^{\lvert V \rvert \times d}$. The set of neighbors for a node $v$ is given by $\mathcal{N}(v)$. The set of overlapping neighbors between two nodes $a$ and $b$, i.e., the common neighbors (CNs), is expressed by $\mathcal{N}^{\text{CN}}_{(a, b)}$. We further denote the set of nodes that are 1-hop neighbors of only one of $a$ or $b$ as $\mathcal{N}^{1}_{(a, b)}$ and the nodes that are ${>}1$-hop from both nodes as $\mathcal{N}^{>1}_{(a, b)}$. Lastly, the personalized pagerank (PPR) score for a root node $v$ and an arbitrary node $u$ is given by $\text{ppr}(v, u)$. 


\section{The Proposed Framework} \label{sec:framework}

\begin{figure*}[t]
    \centering      
       \includegraphics[width=1\linewidth]{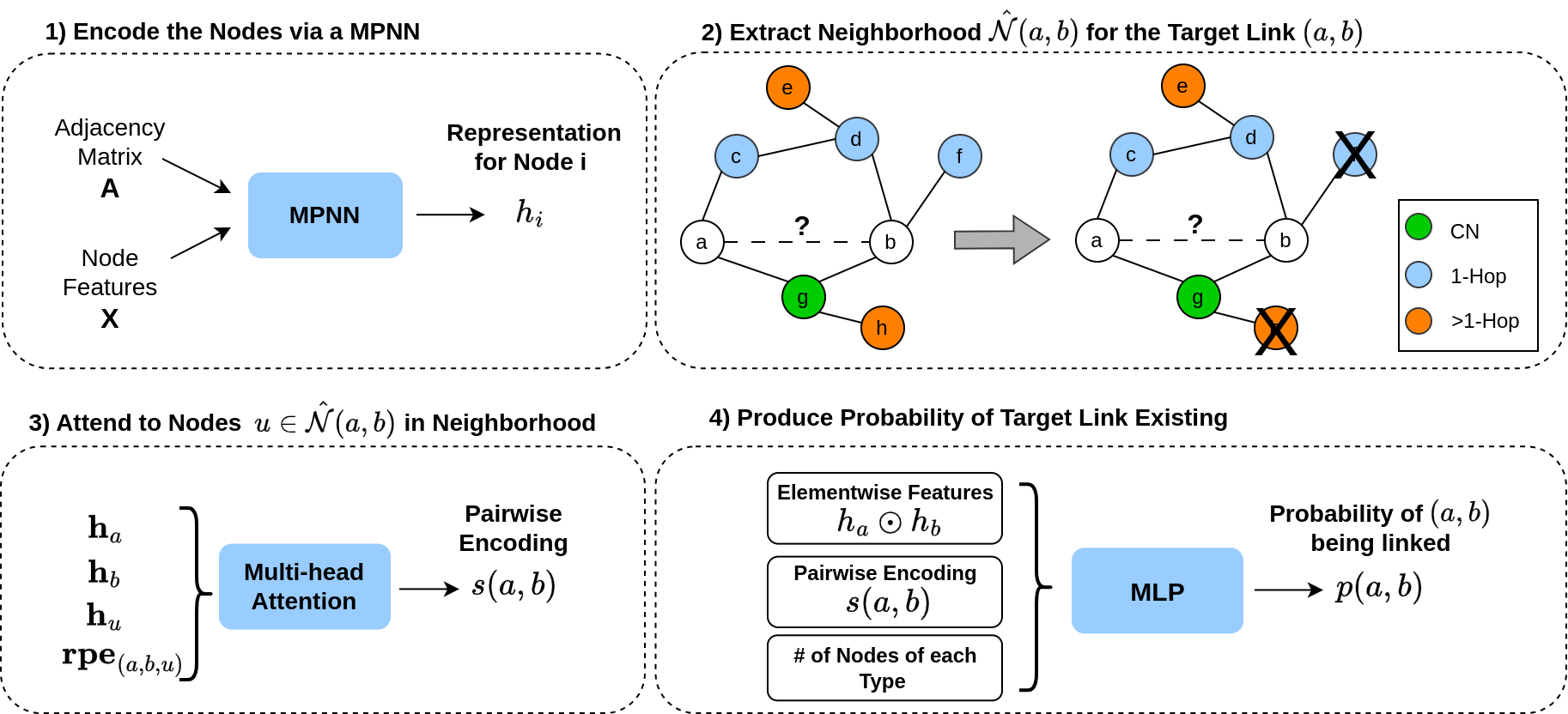}

    \caption{An overview of LPFormer. {\bf (1)} Encode the nodes via a MPNN. {\bf (2)} For a given target link, we determine which nodes to attend to ($\hat{\mathcal{N}}(a, b)$) via the PPR-based thresholding technique in Eq.~\eqref{eq:ppr_threshold}. {\bf (3)} The pairwise encoding is computed by attending to each node, $u \in \hat{\mathcal{N}}(a, b)$ using the feature and relative positional encoding $\mathbf{\text{rpe}}_{(a, b, u)}$. {\bf (4)} The pairwise encoding, node representations, and counts of different node types are concatenated and used to compute the final probability of the target link existing. 
   }

    \label{fig:framework}
\end{figure*}

In Section~\ref{sec:intro}, we highlighted the importance of adaptively modeling multiple types of LP factors. However, current methods that use pairwise encodings, i.e., DP-MPNNs, struggle to appropriately achieve this goal. This is due to two issues: {\bf(1)} They only attempt to model a subset of the potential LP factors (e.g., only local structural information), limiting their ability to model multiple factors. {\bf (2)} They use a one-size-fits-all approach in regard to pairwise encoding, using the same combination of LP factors for each target link. These issues strongly limit the potential of such methods to properly model a variety of different target links. To overcome these problems, we propose {\bf LPFormer}, a new transformer-based method that can adaptively customize the pairwise information for each target link by considering a variety of different LP factors in an efficient manner.

\subsection{A General View of Pairwise Encodings} \label{sec:gen_pairwise}

Recent MPNNs for LP use a decoupled strategy to include the pairwise information~\cite{chamberlain2022graph, ncn, yun2021neo}. These methods, DP-MPNNs, predict the existence of a link $(a, b)$ via both the node representations and a pairwise encoding $s(a, b)$. They follow the formulation below:
\begin{align} \label{eq:gnn_pairwise}
    &H = \text{MPNN}(A, X), \nonumber \\
    &p(a, b) = \sigma \left( \text{MLP} \left( \mathbf{h}_a \odot \mathbf{h}_b \concat s(a, b) \right) \right),
\end{align}
where $h_i$ is the representation of node $i$ encoded by the MPNN. Various DP-MPNNs adopt different ways to model the pairwise encoding. For example, NCN~\cite{ncn} models the pairwise encoding $s(a,b)$ as the summation of the node representations of the CNs. 
The definitions of $s(a,b)$ for other prominent DP-MPNNs can be found in Appendix~\ref{sec:app_pair_encoding}. 
The pairwise encodings in these existing methods are typically manually selected or extracted from the graph, which limits the LP factors they can cover. For example, $s(a,b)$ in NCN and NCNC only capture the local structural information. BUDDY~\cite{chamberlain2022graph} ignores the node features when computing the pairwise encoding. To flexibly model multiple types of LP factors, we propose a general formulation for pairwise encodings as follows,
\begin{equation} \label{eq:gen_pairwise}
    s(a, b) = \sum_{u \in \mathcal{V}} w(a, b, u) \odot h(a, b, u),
\end{equation}
where $w(a, b, u)$ measures the importance of node $u$ to $(a, b)$, and $h(a, b, u)$ is the encoding of node $u$ relative to $(a,b)$. By considering {\it which} nodes should be considered for $(a, b)$ and {\it how} they are related to the node pair, Eq.~\eqref{eq:gen_pairwise} can model different LP factors by manually defining $w(a, b, u)$ and $h(a, b, u)$. In particular, we demonstrate how the heuristic methods corresponding to different LP factors can fit into this framework.

{\bf Common Neighbors (CNs)}~\cite{newman2001clustering}: CNs considers the local structural information and is defined for a pair of nodes $(a, b)$  as $\mathcal{N}^{\text{CN}}_{(a, b)} = \mathcal{N}(a) \cap \mathcal{N}(b)$. Eq.~\eqref{eq:gen_pairwise} is equal to the CNs when $h(a, b, u)=1$ and:
\begin{equation} \label{eq:cn_weight}
    w(a, b, u) = \left\{\begin{array}{ll}
                            1, &\text{when } u \in \mathcal{N}(a) \cap \mathcal{N}(b) \\
                            0, &\text{else } 
                            \end{array}\right\}. 
\end{equation}

{\bf Katz Index}~\cite{katz1953new}: The Katz index models the global structural information. It is defined as weighted summation of the number of paths of different lengths connecting $a$ and $b$ and a decay weight $\beta \in [0, 1]$,
\begin{equation}
    \text{Katz}(a, b) = \sum_{l=1}^{\infty} \beta^l A_{a, b}^l. \nonumber
\end{equation}
This is equivalent to Eq.~\eqref{eq:gen_pairwise} where $w(a, b, u) = \sum_{l=1}^{\infty} \beta^l e_a^T A^l$ and
\begin{equation} \label{eq:b_equals_u1}
     h(a, b, u) = \left\{\begin{array}{ll}
                            e_b^T, &\text{when } u=b \\
                            \mathbf{0}, &\text{else } 
                            \end{array}\right\}, \nonumber
\end{equation}
where $e_i \in \mathbb{B}^{\lvert \mathcal{V} \rvert}$ is a one-hot vector for a node $i$.

{\bf Feature Similarity}: The feature similarity of the pair of nodes $(a, b)$ is expressed by $\text{dis}(\mathbf{x}_a, \mathbf{x}_b)$ where $\mathbf{x}_a$ are the node features of node $a$ and $\text{dis}(\cdot)$ is a distance function (e.g., euclidean distance). This can be rewritten as Eq.~\eqref{eq:gen_pairwise} by substituting $w(a, b, u) = \text{dis}(\mathbf{x}_a, \mathbf{x}_u)$ and $h(a, b, u) = e_b^T$.

These examples demonstrate that the general formulation can indeed model many different LP factors including local and global structural information and feature proximity.
We further show in Appendix~\ref{sec:app_gen_formula} that Eq.~\eqref{eq:gen_pairwise} can model a variety of additional LP factors including RA~\cite{zhou2009predicting}, the pairwise encodings used in NCN/NCNC~\cite{ncn} and Neo-GNN~\cite{yun2021neo}. However, fitting these methods into the formulation in Eq.~\eqref{eq:gen_pairwise} requires manually defining both $w(a, b, u)$ and $h(a, b, u)$. This constrains the information represented by $s(a, b)$ based on the choice of design. Motivated by this, in the next section we introduce our method that does not rely on a handcrafting both $w(a, b, u)$ and $h(a, b, u)$.

\subsection{Modeling Pairwise Encodings via Attention} \label{sec:att}

In Section~\ref{sec:gen_pairwise}, we introduced a general formulation for pairwise encodings in~Eq.~\eqref{eq:gen_pairwise}, which is able to capture a variety of different LP factors. However, it requires manually defining both terms in the equation. This limits our ability to customize the pairwise information to each target link. As such, we further aim to move beyond a one-size-fits-all pairwise encoding, and enable the model to produce customized pairwise encoding for each target link. This allows the model to handle more realistic graphs that often contain multiple prominent LP factors for different target links as shown in~\cite{mao2023revisiting}.  


In particular, we consider the following question: {\it How can we model Eq~\eqref{eq:gen_pairwise} such that it can customize the used LP factors to each target link?} We consider parameterizing both $w(a, b, u)$ and $h(a, b, u)$. This allows us to learn how to personalize them to each target link. To achieve this, we leverage softmax attention~\cite{bahdanau2015neural}. This is due to its ability to dynamically learn the relevance of different nodes to the target link.  As such, for multiple target links, it can emphasize the contributions of different nodes, thereby flexibly modeling different LP factors. 
We note that since the attention is between different sequences (i.e., a target link and nodes), it can be considered a form of cross attention~\cite{vaswani2017attention}. 

To enhance the adaptability of the pairwise encoding for various links, it is essential to incorporate various types of information. This allows the attention mechanism to discern and prioritize relevant information for each target link, facilitating the effective modeling of diverse LP factors. 
In particular, we consider two types of information. The first is the {\bf feature information}. This includes the feature representation of both nodes in the target link and the node being attended to. The node features are included due to their role in link formation and relationship to structural information~\cite{murase2019structural}. Second, we consider the {\bf relative positional information}. The relative positional information reflects the relative position in the graph of a node $u$ to the target link $(a, b)$ in the local and global structural context. Due to the importance of local and global structural information~\cite{dong2017structural, huang2015triadic}, it is vital to properly encode both. By including both the structural and feature information, we are able to cover the space of potential LP factors (see Section~\ref{sec:related_work}).  

We denote the feature representation of a node $u$ as $\mathbf{h}_u$ and the relative positional encoding (RPE) as  $\mathbf{rpe}_{(a, b, u)}$. The node importance  $w(a, b, u)$ is modeled via attention as follows:
\begin{align} \label{eq:att1}
    &\tilde{w}(a, b, u) = \phi \left(\mathbf{h}_a, \mathbf{h}_b, \mathbf{h}_u, \: \mathbf{rpe}_{(a, b, u)} \right), \nonumber \\
    &w(a, b, u) = \frac{\text{exp}(\tilde{w}(a, b, u))}{\sum_{v \in \bar{\mathcal{V}}(a, b)}\text{exp}(\tilde{w}(a, b, u))},
\end{align}
where $\bar{\mathcal{V}}(a, b) = \mathcal{V} \setminus\{a, b\}$. The attention weight $w(a, b, u)$ can be considered as the impact of a node $u$ on $(a, b)$ relative to all nodes in $\mathcal{G}$. This allows the model to emphasize different LP factors for each target link. The node encoding $h(a, b, u)$ includes the features of node $u$ in conjunction with the RPE and is defined as: 
\begin{equation} \label{eq:att2}
    h(a, b, u) = \mathbf{W} \left[\mathbf{h}_u \: \concat \mathbf{rpe}_{(a, b, v)} \right].
\end{equation}
By substituting Eq.~\eqref{eq:att1} and Eq.~\eqref{eq:att2} into Eq.~\eqref{eq:gen_pairwise} we can compute the pairwise information $s(a, b)$. We further define $\phi( \cdot )$ in Eq.~\eqref{eq:att1} as the  GATv2~\cite{gatv2} attention mechanism. The detailed formulation is given in Appendix~\ref{sec:app_attention}. The feature representations $\mathbf{h}_i$ are computed via a MPNN. We use GCN~\cite{kipf2017semi} in this work. However, it is unclear how to properly encode the RPE of a node $u$ relative to $(a, b)$, $\mathbf{rpe}_{(a, b, u)}$. We aim to design the RPE to capture both the local and global structural relationship between the node and target link while also being efficient to calculate. In the next section, we discuss our solution for modeling $\mathbf{rpe}_{(a, b, u)}$. 

\subsection{PPR-Based Relative Positional Encodings} \label{sec:rpe} 



In this section, we introduce our strategy for computing the RPE of a node $u$ relative to a target link $(a, b)$.
Intuitively, we want the RPE to reflect the positional relationship between $u$ and $(a, b)$ such that different types of information (i.e., local vs. global) are encoded differently. Using Figure~\ref{fig:example} as an example, since node 3 is a CN of (source, 5) we expect it to have a much different relationship to the target link than node 6, which is a 2-hop neighbor of both nodes.
An enticing option is to use the double radius node labeling (DRNL) trick introduced by~\citet{seal}. However, \citet{chamberlain2022graph} have shown it to be prohibitively expensive to calculate for larger graphs. Furthermore, existing RPEs are typically infeasible to calculate on larger graphs as they often rely on pairwise distances or the eigenvectors of the Laplacian~\cite{gps}.  

As such, we seek an RPE that can both distinguish the relationship of different nodes to the target link while also being efficient to calculate. To motivate our RPE design, we draw inspiration from the following Proposition.

\begin{restatable}[]{proposition}{fta}
\label{th:ppr}
    Consider a target link $(a, b)$ and a node $u \in \mathcal{V} \setminus \{a, b\}$. The PPR~\cite{pagerank} score of a root node $i$ and target node $j$ with teleportation probability $\alpha$ is denoted by  $\text{ppr}(i, j)$. 
    Let $r_{a}^k (u)$ be the probability of a walk of length $k$ beginning at node $a$ and terminating at $u$. We define $r_{a, b}^k (u) := r_{a}^k (u) + r_{b}^k (u)$. We also define a weight $\gamma^k:=\alpha (1-\alpha)^{k}$ for all walks of length $k$. 
    The PPR scores, \(ppr(a,u)\) and \(ppr(b,u)\), along with the random walk probabilities of disparate lengths, are interconnected through the following  relationship.
    \begin{equation}
        \Gamma(a, b, u) = \text{ppr}(a, u) + \text{ppr}(b, u) = \sum_{k=0}^{\infty}  \gamma^{k} r_{a, b}^k (u).
    \end{equation}
\end{restatable}

The detailed proof is given in Appendix~\ref{sec:proof1}. From Proposition~\ref{th:ppr}, we can make the following observations: (1) The PPR scores encode the weighted sum of the probabilities of different length random walks connecting two nodes. (2) Walks of shorter length are given higher importance, as evidenced by the dampening factor $\gamma^k = \alpha (1-\alpha)^{k}$ which decays with the increase in $k$. These observations imply that -- {\bf a larger value of $\Gamma(a, b, u)$ correlates with the existence of many shorter walks connecting node $u$ to the both nodes in the target link $(a, b)$.} 

Therefore, the PPR scores can be used as an intuitive and useful method to understand the structural relationship between node $u$ and both nodes in the target link $(a, b)$. If both scores, $\text{ppr}(a, u)$ and $\text{ppr}(b, u)$, are high, there exists a high probability that many shorter walks connect $u$ to both nodes in the target link. This implies that node $u$ has a stronger impact on the nodes in the target link. On the other hand, if both PPR scores are low, there is likely very little relationship between $u$ and the target link. This allows for a convenient way of differentiating how a node structurally relates to the target link. Furthermore, we note that the PPR matrix can be efficiently pre-computed using the algorithm introduced by~\citet{andersen2006local}, allowing for easy computation and use.

Following this idea, to calculate the RPE of a node $u$, we use the PPR scores of a node $u$ relative to both nodes in the target link $(a,b)$. Instead of considering the sum of PPR scores as in Proposition~\ref{th:ppr}, we further parameterize $\Gamma(\cdot)$ via an MLP, 
\begin{equation} \label{eq:ppr1}
    \mathbf{rpe}_{(a, b, u)} = \text{MLP} \left( \text{ppr}(a, u), \text{ppr}(b, u) \right) .
\end{equation}
By introducing learnable parameters to $\Gamma(\cdot)$, it allows for the model learn the importance of individual PPR scores and how they interact with each other. 
To ensure that Eq.~\eqref{eq:ppr1} is invariant to the order of the nodes in the target link, i.e., $(a, b)$ and $(b, u)$, we further set the RPE to be equal to the summation of the representations given by both $(a, b)$ and $(b, a)$:
\begin{equation} \label{eq:ppr2}
    \mathbf{\overline{rpe}}_{(a, b, u)} = \mathbf{rpe}_{(a, b, u)} + \mathbf{rpe}_{(b, a, u)}.
\end{equation}
However, a concern with Eq.~\eqref{eq:ppr2} is that it is not guaranteed to be able to distinguish certain types of nodes from each other. 
For example, it is necessary to clearly distinguish CNs from other nodes due to their important role in link formation~\cite{newman2001clustering}. To overcome this issue, we fit three separate MLPs for when $u$ is a: CN of $(a, b)$, a 1-hop neighbor of either $a$ and $b$, and a ${>}1$-hop neighbor of both $a$ and $b$. This ensures that we can properly distinguish between these three types of nodes. We verify the effectiveness of this design in~Section~\ref{sec:ablation}. Lastly, we note that while other work~\cite{mialon2021graphit, li2020distance} has considered the use of random-walk based positional encodings, they are only designed for use on the node-level and are unable to be used for link-level tasks like LP.

\subsection{Efficiently Attending to the Graph Context} \label{sec:nodes_attending}

The proposed attention mechanism in Section~\ref{sec:att} attends to all nodes in the graph, sans those in the link itself. This makes it difficult to scale to large graphs. Motivated by selective~\cite{maruf2019selective} and sparse~\cite{correia2019adaptively} attention, we opt to attend to only a small portion of the nodes.


At a high level, we are interested in determining a subset of nodes $\hat{\mathcal{N}}(a, b) \in \mathcal{V}$ to attend to for the target link $(a, b)$. Our goal is to choose the set of nodes $\hat{\mathcal{N}}(a, b)$ such that they are {\bf(a)} few in number to improve scalability and {\bf (b)} provide important contextual information to the pair $(a, b)$ to best learn the pairwise information. This can be achieved by only considering all nodes where the importance of the node $u$ to the target link $(a, b)$ is considered high. Formally, we can write this as the following where $\mathcal{I}(a, b, u)$ is a function that denotes the importance of a node $u$ to the target link $(a, b)$:
\begin{equation} \label{eq:ppr_thresh_1}
    \hat{\mathcal{N}}(a, b) = \{u \in \mathcal{V}\setminus\{a, b\} \; | \; \mathcal{I}(a, b, u) > \eta \}.
\end{equation}
The threshold $\eta$ allows us to distinguish those nodes that are sufficiently important to the target link.
This allows for a simple and efficient way of determining the set $\hat{\mathcal{N}}(a, b)$. However, {\it what do we use to model the importance $\mathcal{I}(a, b, u)$?} For ease of optimization and better efficiency, we avoid parameterizing the function $\mathcal{I}(a, b, u)$. Instead, we want to choose a metric such that can properly serve as a proxy for the importance of a node $u$ to $(a, b)$ while also being concentrated in a small subset of nodes. Such a metric will allow Eq.~\eqref{eq:ppr_thresh_1} to choose a small but influential set of nodes to attend to.

A measure that satisfies both criteria is Personalized Pagerank (PPR)~\cite{pagerank}. In Section \ref{sec:rpe} we discussed that the PPR score can serve as a good tool to model the influence of a one node on another. Furthermore, existing work~\cite{gleich2015localization, nassar2015strong, andersen2006local} shows that the PPR scores tend to be highly localized in a small subset of nodes. Therefore by making $\mathcal{I}(a, b, u)$ contingent on the PPR scores of $(a, u)$ and $(b, u)$ we can extract a small but important set of nodes to attend to for the target link.

Following this idea, for a target link $(a, b)$, we keep all nodes whose PPR score is above some threshold $\eta$ relative to both nodes in the target link. As such, we only keep a node $u$ if it is related in some capacity to at least one of the nodes in the target link. Similarly to Section~\ref{sec:rpe}, we treat CN, 1-Hop, and ${>}1$-Hop nodes differently by applying a different threshold for them. The filtered node set for each category of nodes is given by: 
\begin{equation} \label{eq:ppr_threshold}
    \hat{\mathcal{N}}^{\pi}_{(a, b)} =  \{u \in \mathcal{N}^{\pi}_{(a, b)} \: | \: \text{ppr}(a, u) > \eta^{\pi}, \: \text{ppr}(b, u) > \eta^{\pi}  \},
\end{equation} 
where $\hat{\mathcal{N}}^{\pi}_{(a, b)}$ is the filtered node set for all nodes of the type~$\pi \in \{\text{CN}, 1{-}\text{Hop}, {>}1{-}\text{Hop} \}$ and $\eta^{\pi}$ is the corresponding PPR threshold. We note that while other work~\cite{bojchevski2020scaling, ying2018graph} has used PPR to filter the nodes on the {\it node-level}, no existing work has done so on the {\it link-level}.

We corroborate this design by demonstrating that LPFormer can achieve SOTA performance in LP (Section~\ref{sec:main_results}) while achieving a faster runtime than the second-best method, NCNC~\cite{ncn}, on denser graphs (Section~\ref{sec:runtime}). This is despite the fact that LPFormer can attend to a wider variety of nodes. We further show in Section~\ref{sec:exp_thresh} that the performance is stable with regards to the values of $\eta$ chosen, allowing us to easily choose a proper threshold on any dataset.

\subsection{LPFormer} \label{sec:lpformer}

We now define the overall framework -- LPFormer. The overall procedure is given in Figure~\ref{fig:framework}: {\bf (1)} We first learn node representations from the input adjacency and node features via an MPNN. We note that this step is agnostic to the target link. {\bf (2)} For a target link $(a, b)$ we extract the nodes to attend to, i.e. $\hat{\mathcal{N}}(a, b)$. This is done via the PPR thresholding technique defined in Section~\ref{sec:nodes_attending}. {\bf (3)} We apply $L$ layers of attention, using the mechanism defined in Section~\ref{sec:att}. The output is the pairwise encoding $s(a, b)$. {\bf (4)} We generate the prediction of the target link using three types of information: the element-wise product of the node representation, the pairwise encoding, and the number of CN, 1-Hop, and $>$1-Hop nodes identified by Eq.~\eqref{eq:ppr_threshold}. The score function is given by:
{\small
\begin{equation}
    p(a, b) = \sigma \left( \text{MLP} \left( \mathbf{h}_a \odot \mathbf{h}_b \concat s(a, b) \concat \lvert \hat{\mathcal{N}}^{\text{CN}}_{(a, b)} \rvert \concat \lvert \hat{\mathcal{N}}^{1}_{(a, b)} \rvert \concat \lvert \hat{\mathcal{N}}^{>1}_{(a, b)} \rvert \right) \right)
\end{equation}
}We demonstrate in Section~\ref{sec:ablation} that the inclusion of the node counts is helpful, as it provides complementary information to the pairwise encoding.

\section{Experiments} \label{sec:experiments}

\begin{table*}[t]
\centering

 \caption{Dataset statistics. The split ratio is the \% of samples for train/validation/test. 
 }
    \begin{tabular}{cccccccc}
    \toprule
     & Cora & Citeseer & Pubmed & ogbl-collab & ogbl-ddi & ogbl-ppa & ogbl-citation2 \\
     \midrule
    \#Nodes & \multicolumn{1}{r}{2,708} & \multicolumn{1}{r}{3,327} & \multicolumn{1}{r}{18,717} & \multicolumn{1}{r}{235,868} & \multicolumn{1}{r}{4,267} & \multicolumn{1}{r}{576,289} & \multicolumn{1}{r}{2,927,963} \\
    \#Edges & \multicolumn{1}{r}{5,278} & \multicolumn{1}{r}{4,676} & \multicolumn{1}{r}{44,327} & \multicolumn{1}{r}{1,285,465} & \multicolumn{1}{r}{1,334,889} & \multicolumn{1}{r}{30,326,273} & \multicolumn{1}{r}{30,561,187} \\
    Split Ratio& \multicolumn{1}{r}{85/5/10}&\multicolumn{1}{r}{85/5/10} &\multicolumn{1}{r}{85/5/10} &\multicolumn{1}{r}{92/4/4} & \multicolumn{1}{r}{80/10/10} & \multicolumn{1}{r}{70/20/10}&\multicolumn{1}{r}{98/1/1}\\
    \bottomrule
    \end{tabular}
    \label{table:app_data}
\end{table*}

\begin{table*}[t!]
    \centering
    \caption{Results on benchmark datasets. OOM is an out of memory error. Colored are the results ranked \colorfirst{first}, \colorsecond{second}, and  \colorthird{third}.}  \label{tab:main_results}

    \begin{tabular}{lcccccc|c}
    \toprule
         &
         \textbf{Cora} &  
         \textbf{Citeseer} & 
         \textbf{Pubmed} &
         \textbf{ogbl-collab} &
         \textbf{ogbl-ppa} &
         \textbf{ogbl-citation2} &
         \textbf{Mean Rank} 
         \\
          \cmidrule{2-7}
          Metric &
          MRR &
          MRR & 
          MRR &
          H@50 &
          H@100 &
          MRR & 
         \\ 
         
         \midrule
          
         \textbf{CN} & 
         20.99${\scriptstyle \pm 0.00}$& 
         28.34${\scriptstyle \pm 0.00}$& 
         14.02${\scriptstyle \pm 0.00}$&
         56.44${\scriptstyle \pm 0.00}$&
         27.65${\scriptstyle \pm 0.00}$&
         51.47${\scriptstyle \pm 0.00}$&
         11.0
         \\

        \textbf{AA} & 
        31.87${\scriptstyle \pm 0.00}$&
        29.37${\scriptstyle \pm 0.00}$&
        16.66${\scriptstyle \pm 0.00}$&
        64.35${\scriptstyle \pm 0.00}$&
        32.45${\scriptstyle \pm 0.00}$&
        51.89${\scriptstyle \pm 0.00}$&
        8.5
        \\

        \textbf{RA} &
        30.79${\scriptstyle \pm 0.00}$ &
        27.61${\scriptstyle \pm 0.00}$&
        15.63${\scriptstyle \pm 0.00}$& 
        64.00${\scriptstyle \pm 0.00}$&
        49.33${\scriptstyle \pm 0.00}$ & 
        51.98${\scriptstyle \pm 0.00}$&
        8.7
        \\ \midrule
          
        \textbf{GCN} & 
        32.50${\scriptstyle \pm 6.87}$& 
        50.01${\scriptstyle \pm 6.04}$&
        19.94${\scriptstyle \pm 4.24}$& 
        44.75${\scriptstyle \pm 1.07}$&
        18.67${\scriptstyle \pm 1.32}$&
        84.74${\scriptstyle \pm 0.21}$&
        8.0
         \\
        \textbf{SAGE} & 
        \colorsecond{37.83${\scriptstyle \pm 7.75}$}& 
        47.84${\scriptstyle \pm 6.39}$& 
        22.74${\scriptstyle \pm 5.47}$& 
        48.10${\scriptstyle \pm 0.81}$& 
        16.55${\scriptstyle \pm 2.40}$&
        82.60${\scriptstyle \pm 0.36}$&
        7.7
        \\ 
        \textbf{GAE} & 
        29.98${\scriptstyle \pm 3.21}$& 
        \colorthird{63.33${\scriptstyle \pm 3.14}$}& 
        16.67${\scriptstyle \pm 0.19}$& 
        OOM & 
        OOM & 
        OOM &
        NA
        \\ 
        \midrule

        \textbf{SEAL} & 
        26.69${\scriptstyle \pm 5.89}$& 
        39.36${\scriptstyle \pm 4.99}$ & 
        \colorthird{38.06${\scriptstyle \pm 5.18}$}& 
        64.74${\scriptstyle \pm 0.43}$& 
        48.80${\scriptstyle \pm 3.16}$& 
        {87.67${\scriptstyle \pm 0.32}$} &
        6.2
        \\ 
        
         \textbf{NBFNet} & 
         \colorthird{37.69${\scriptstyle \pm 3.97}$}& 
         38.17${\scriptstyle \pm 3.06}$&
         \colorfirst{44.73${\scriptstyle \pm 2.12}$}& 
         OOM&
         OOM&
         OOM&
         NA
         \\  

         \midrule

        \textbf{Neo-GNN} & 
        22.65${\scriptstyle \pm 2.60} $ &
        53.97${\scriptstyle \pm 5.88}$ &
        31.45${\scriptstyle \pm 3.17}$ &
        57.52${\scriptstyle \pm 0.37}$& 
        49.13${\scriptstyle \pm 0.60}$& 
        87.26${\scriptstyle \pm 0.84}$ &
        7.0
        \\ 
               
         \textbf{BUDDY} & 
         26.40${\scriptstyle \pm 4.40}$&
         59.48${\scriptstyle \pm 8.96}$&
         23.98${\scriptstyle \pm 5.11}$&
         \colorthird{65.94${\scriptstyle \pm 0.58}$}& 
         {49.85${\scriptstyle \pm 0.20}$}& 
         {87.56${\scriptstyle \pm 0.11}$} &
         \colorthird{5.7}
         \\

         \textbf{NCN} &
         32.93${\scriptstyle \pm 3.80}$&
         54.97${\scriptstyle \pm 6.03}$&
         35.65${\scriptstyle \pm 4.60}$&
         64.76${\scriptstyle \pm 0.87}$&
         \colorthird{61.19${\scriptstyle \pm 0.85}$}&
         \colorthird{88.09${\scriptstyle \pm 0.06}$}&
         \colorsecond{3.8}
         \\
         
         \textbf{NCNC} &
         29.01${\scriptstyle \pm 3.83}$&
         \colorsecond{64.03${\scriptstyle \pm 3.67}$}&
         25.70${\scriptstyle \pm 4.48}$&
         \colorsecond{66.61${\scriptstyle \pm 0.71}$}& 
         \colorsecond{61.42${\scriptstyle \pm 0.73}$}& 
         \colorsecond{89.12${\scriptstyle \pm 0.40}$}&
         \colorsecond{3.8}
         \\

         \midrule 
         \textbf{LPFormer} & 
         \colorfirst{39.42${\scriptstyle \pm 5.78}$}& 
         \colorfirst{65.42${\scriptstyle \pm 4.65}$}&
         \colorsecond{40.17${\scriptstyle \pm 1.92}$}& 
         \colorfirst{68.14${\scriptstyle \pm 0.51}$}&
         \colorfirst{63.32${\scriptstyle \pm 0.63}$}&
         \colorfirst{89.81${\scriptstyle \pm 0.13}$} &
         \colorfirst{1.2}
         \\ \bottomrule
\end{tabular}
\end{table*}

In this section, we conduct extensive experiments to validate the effectiveness of LPFormer. Specifically, we attempt to answer the following questions: {\bf (RQ1)} Can LPFormer consistently outperform baseline methods on a variety of different benchmark datasets? {\bf (RQ2)} Is LPFormer able to model a variety of different LP factors? {\bf (RQ3)} Can LPFormer be run efficiently on large dense graphs? We further conduct studies ablating each component of our model and analyzing the effect of the PPR-based threshold on performance.

\subsection{Experimental Settings}

\hspace{\parindent}{\bf Datasets}. We include Cora, Citeseer, and Pubmed~\cite{planetoid} and ogbl-collab, ogbl-ppa, ogbl-ddi, and ogbl-citation2~\cite{ogb}. Furthermore, for Cora, Citeseer, and Pubmed we experiment under a single fixed split (see Appendix~\ref{sec:planetoid_splits} for further discussion). The detailed statistics for each dataset are shown in Table~\ref{table:app_data}.

{\bf Baseline Models}. We compare LPFormer against a wide variety of baselines including: CN~\cite{newman2001clustering}, AA~\cite{adamic2003friends}, RA~\cite{zhou2009predicting}, GCN~\cite{kipf2017semi}, SAGE~\cite{hamilton2017inductive}, GAE~\cite{kipf2016variational}, SEAL~\cite{seal}, NBFNet~\cite{nbfnet}, Neo-GNN~\cite{yun2021neo}, BUDDY~\cite{chamberlain2022graph}, and NCNC~\cite{ncn}. Results on Cora, Citeseer, and Pubmed are taken from~\citet{li2023evaluating}. Results for the heuristic methods are from~\citet{ogb}. All other results are either from their respective study or~\citet{chamberlain2022graph}.


{\bf Hyperparameters}: The learning rate is tuned from $\{1e^{-3}, 5e^{-3}\}$, the decay from $\{0.95, 0.975, 1 \}$, and the dropout from $[0, 0.7]$, and the weight decay from $\{0, 1e^{-4}, 1e^{-7} \}$. The size of the hidden dimension is set to 64 for ogbl-ppa and ogbl-citation2, 128 for Cora, Pubmed, and ogbl-collab, and 256 for Citeseer. Lastly, the PPR threshold is tuned from $\{1e^{-2}, 1e^{-3}, 1e^{-4} \}$.

{\bf Evaluation Metrics}. Each positive target link is evaluated against a set of given negative links. The rank of the positive link among the negatives is used to evaluate performance. The two types of metrics that are used to evaluate this ranking are Hits@K and MRR. For the OGB datasets we use the metric used in the original study. This includes Hits@50 for ogbl-collab, Hits@100 for ogbl-ppa and MRR for ogbl-citation2. For Cora, Citeseer, Pubmed we follow~\citet{li2023evaluating} and use MRR. Lastly, the same set of negative links is used for all positive links except on ogbl-citation2, where~\cite{ogb} provides a customized set of 1000 negatives for each individual positive link.

\subsection{Main Results} \label{sec:main_results}

We present the results of LPFormer compared with baselines on multiple benchmark datasets. Note that we omit ogbl-ddi from the main results due to recent issues discovered by~\citet{li2023evaluating} (see Appendix~\ref{sec:ddi} for more details). The results are shown in Table~\ref{tab:main_results}. We observe that LPFormer can achieve SOTA performance on 5/6 datasets, significantly outperforming other baselines. Moreover, LPFormer is also the most consistent of all the methods, achieving strong performance on all datasets. This is as opposed to previous SOTA methods, NCNC and BUDDY, which tend to struggle on Cora and Pubmed. We attribute the consistency of LPFormer to the flexibility of our model, allowing it to customize the LP factors needed to each link and dataset.

\subsection{Performance by LP Factor} \label{sec:exp_factor}

In this section, we measure the ability of LPFormer to capture a variety of different LP factors. To measure this, we identify all positive target links {\bf when there is only one dominant LP factor}. For example, one group would contain all target links where the only dominant factor is the local structural information. We focus on links that correspond to one of the three groups identified in~\cite{mao2023revisiting}: local structural information, global structural information, and feature proximity. 

We identify these groups by using popular heuristics as proxies for each factor. For local structural information, we use CNs~\cite{newman2001clustering}, for global structural information we use PPR~\cite{pagerank} as it's the most computationally efficient of all global methods, and for feature proximity, we use the cosine similarity of the features. Using these heuristics, we determine if only one factor is dominant by comparing the relative score of each heuristic. This is done by first computing the score for each factor $i$ for the target link $(a, b)$ -- $s^i(a, b)$. For each factor, we then compute the score corresponding to the $p$-th percentile among all links,  $\hat{s}^i$. We choose a larger value of $p$ (i.e. 90\%) such that a score $\geq \hat{s}^i$ indicates that a significant amount of pairwise information exists for that factor. For a single target link, we then compare the score of each factor $s^i(a, b)$ to $\hat{s}^i$. If $s^i(a, b) \geq \hat{s}^i$ is true {\bf for only one factor}, this implies that the score for only one factor is ``high''. Therefore there is a notable amount of pairwise information existing for only one factor for the link $(a, b)$. This ensures that only one factor is strongly expressed. If this is true, we then assign the target link $(a, b)$ to factor $i$. Please see Appendix~\ref{sec:app_factor_details} for a more detailed explanation.

We demonstrate the results on Cora, Citeseer, and ogbl-collab in Figure~\ref{fig:exp_factors}. We observe that LPFormer typically performs best for each individual LP factor on all datasets. Furthermore, it is also the most consistently well-performing on each factor as compared to other methods. For example, on Cora the other methods struggle for links that correspond to the feature proximity factor. LPFormer, on the other hand, is able to significantly outperform them on those target links, performing around 33\% better than the second best method. Lastly, we note that most methods tend to perform well on the links corresponding to the global factor, even if they don't explicitly model such information. This is caused by a strong correlation that tends to exist between local and global structural information, often resulting in considerable overlap between both factors~\cite{mao2023revisiting}.
These results show that LPFormer can indeed adapt to multiple types of LP factors, as it can consistently perform well on samples belonging to a variety of different LP factors.
Additional results are given in Appendix~\ref{sec:app_factors_exp}.

\begin{figure*}[t]
\centering
    \begin{subfigure}{.333\textwidth}
      \includegraphics[width=0.999\linewidth]{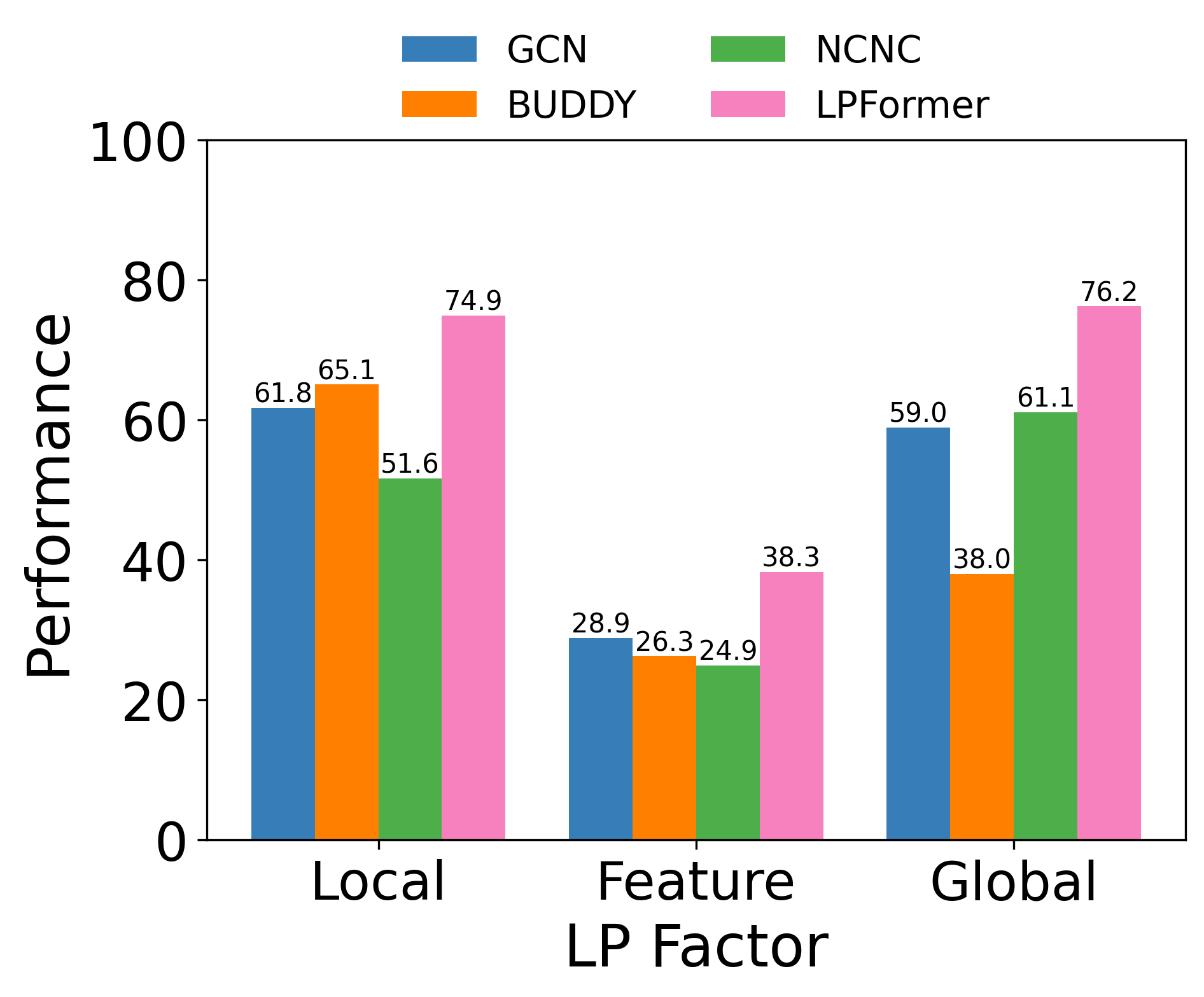}
      \caption{Cora}
      \label{fig:factor_cora}
    \end{subfigure}%
    \begin{subfigure}{.333\textwidth}
      \includegraphics[width=0.999\linewidth]{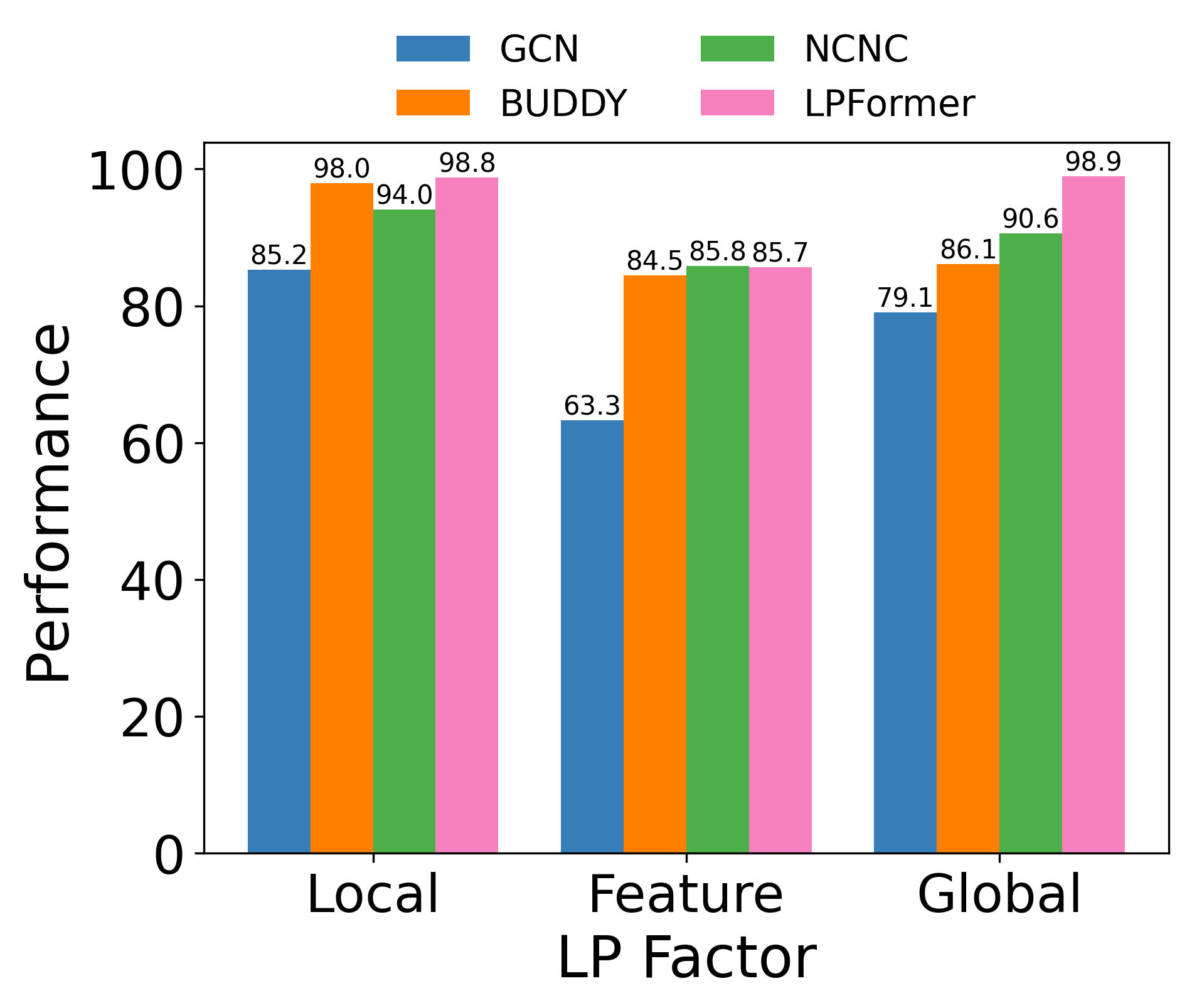}
      \caption{Citeseer}
      \label{fig:factor_citeseer}
    \end{subfigure}%
    \begin{subfigure}{.333\textwidth}
      \includegraphics[width=0.999\linewidth]{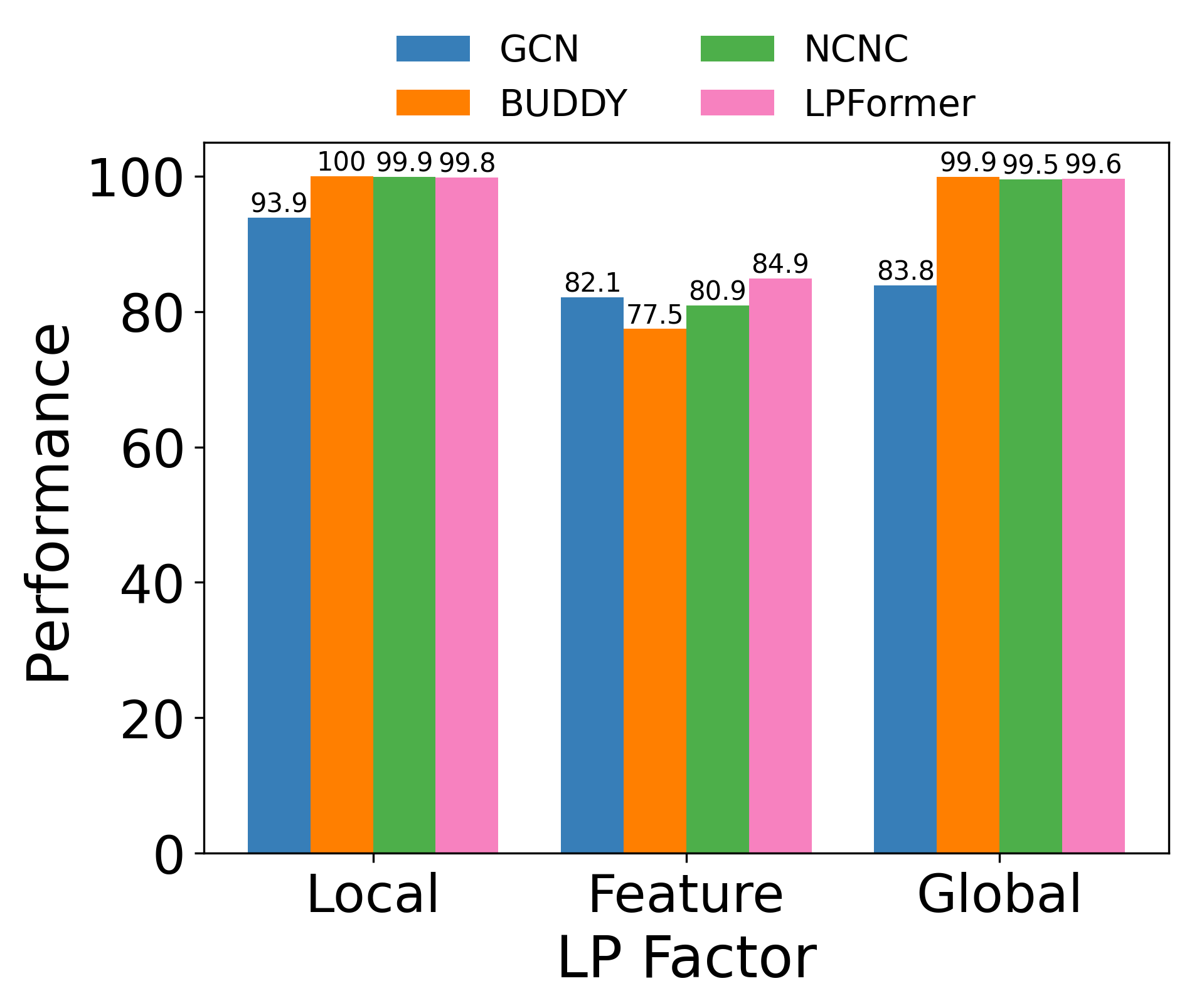}
      \caption{ogbl-collab}
      \label{fig:factor_collab}
    \end{subfigure}%
    \caption{Performance on links that contain one dominant LP factor. Results are on (a) Cora, (b) Citeseer, and (c) ogbl-collab. }

\label{fig:exp_factors}
\end{figure*}

\subsection{Ablation Study} \label{sec:ablation}

We further include an ablation study to verify the effectiveness of the proposed components in LPFormer. In particular, we introduce 6 variants of LPFormer. (a) {\bf w/o Learnable Att}: No attention is learned. As such, we set all attention weights to 1 and remove the RPE. (b) {\bf w/o Features in Att}: We remove the node feature information from the attention mechanism. (c) {\bf w/o RPE in Att}: We remove the RPE from the attention mechanism. (d) {\bf w/o PPR RPE}: We replace the PPR-based RPE with a learnable embedding for each of CN, 1-Hop, and $>$1-Hop nodes. (e) {\bf w/o PPR RPE by Node Type}: We don't fit a separate function for each node type when determining the PPR RPE (see Section~\ref{sec:rpe}). Instead we use one for all nodes. (f) {\bf w/o Counts}: We remove the counts of different nodes from the scoring function. 

The results are shown in Table~\ref{tab:ablation}. We include ogbl-collab, ogbl-ppa, and Citeseer. We observe that ablating a component always decreases the performance. However, the magnitude of the decrease is dataset-dependent. For example, on ogbl-collab, ablating the feature information in the attention marginally affects the performance. However, on ogbl-ppa and Citeseer, removing the feature information results in a large decrease in performance. On the other hand, while removing learnable attention results in a modest decrease on ogbl-ppa, for the other two datasets we see a large drop. This highlights the importance of each component of our framework, as they are each necessary for consistently strong performance across multiple datasets.

\begin{table}[h]
    \small
	\centering
	\caption{Ablation Study on LPFormer}
	\adjustbox{max width=\textwidth}{
		\begin{tabular}{@{}  l | c c c @{}}
			\toprule
			\textbf{Method} & {\bf ogbl-collab} & {\bf ogbl-ppa} & {\bf Citeseer} \\
            \midrule
            w/o Learnable Att & $65.05{\scriptstyle \pm 0.50}$ & $62.77{\scriptstyle \pm 1.03}$ & $56.23{\scriptstyle \pm 1.75}$ \\
            w/o Features in Att & $68.04\scriptstyle \pm 0.79$ & $56.98{\scriptstyle \pm 1.55}$ & $53.40{\scriptstyle \pm 9.30}$ \\
            w/o RPE in Att & $65.26{\scriptstyle \pm 0.56}$ & $61.20{\scriptstyle \pm 0.69}$ & $56.70{\scriptstyle \pm 3.79}$ \\
            w/o PPR RPE & $67.09{\scriptstyle \pm 0.51}$ & $61.91{\scriptstyle \pm 1.22}$ & $51.96 {\scriptstyle \pm 15.2}$  \\
            w/o  PPR RPE by Node Type & $67.95{\scriptstyle \pm 0.54}$ & $62.92{\scriptstyle \pm 1.06}$ & $57.40{\scriptstyle \pm 5.71}$\\ 
            w/o Counts & $67.75{\scriptstyle \pm 0.41}$ & ${44.37 \scriptstyle \pm 1.89}$ & $54.39{\scriptstyle \pm 5.30}$  \\
            \midrule
            LPFormer & $\mathbf{68.14 {\scriptstyle \pm 0.51}}$ & $\mathbf{63.32 {\scriptstyle \pm 0.63}}$ & $\mathbf{65.42\scriptstyle \pm 4.65}$  \\
			\bottomrule
		\end{tabular}}
\label{tab:ablation}
\end{table}

\begin{table}[h]
	\centering
	\caption{Effect of Varying the PPR Thresholds}
	\adjustbox{max width=\linewidth}{
		\begin{tabular}{@{}c | cc | cc }
			\toprule
    		   \multicolumn{1}{c}{\textbf{Threshold}} &
			   \multicolumn{2}{c}{\textbf{ogbl-collab}} &
			   \multicolumn{2}{c}{\textbf{ogbl-citation2}} \\
			    \cmidrule(l){2-3} \cmidrule(l){4-5}  
			   & 1-Hop  & ${>}1{-}\text{Hop}$ & 1-Hop  & ${>}1{-}\text{Hop}$
			   \\ \midrule
               1e-4 & $68.24{\scriptstyle \pm 0.25}$ & $67.73{\scriptstyle \pm 0.65}$ & $89.81{\scriptstyle \pm 0.13}$ & $89.14{\scriptstyle \pm 0.22}$ \\
               1e-2 & $67.60{\scriptstyle \pm 0.31}$ & $68.24{\scriptstyle \pm 0.25}$ & $89.49{\scriptstyle \pm 0.18}$ & $89.81{\scriptstyle \pm 0.13}$ \\
               1    & $67.08{\scriptstyle \pm 0.65}$ & $68.14{\scriptstyle \pm 0.51}$ & $89.49{\scriptstyle \pm 0.16}$ & $89.26{\scriptstyle \pm 0.39}$ \\
			\bottomrule
		\end{tabular}}
  \label{tab:param_study_eps}
\end{table}

\subsection{Effect of the PPR Thresholds} \label{sec:exp_thresh}

We examine the effect of varying the PPR threshold for both 1-Hop and ${>}1{-}\text{Hop}$ nodes as described in Eq.~\eqref{eq:ppr_threshold}. The results for ogbl-collab and ogbl-citation2 are shown in Table~\ref{tab:param_study_eps}. When varying the 1-Hop threshold, we fix the value of the ${>}1{-}\text{Hop}$ threshold to 1e-2 for both datasets.  When varying the ${>}1{-}\text{Hop}$ threshold, we fix the value of the $1${}-Hop threshold to 1e-4 for both datasets.

We can observe that modifying the threshold has little effect on the underlying performance of the model. For both datasets, a value of 1e-2 works well for the ${>}1{-}\text{Hop}$ threshold and 1e-4 works well for the 1-Hop threshold. We typically find that setting both values to 1e-2 provides a good trade-off between performance and efficiency.


    

\subsection{Performance on HeaRT Setting}

We further test the performance of our method on the HeaRT~\cite{li2023evaluating} evaluation setting, which considers a more realistic and difficult evaluation setting for link prediction. This is done by introducing a much harder and more realistic set of negative samples during evaluation.  \citet{li2023evaluating} observe that this results in a large decrease in performance on all datasets. Furthermore, compared to the original evaluation setting, MPNNs designed specifically for link prediction are often outperformed by heuristics or other MPNNs.

The full results can be found in Table~\ref{tab:heart_full}. We observe that LPFormer performs considerably better than all other models. For instance, the mean rank of LPFormer is 3.1x better than the 2nd best-performing model, NCN. This indeed shows the advantage of LPFormer, as it can consistently achieve extraordinary performance across all datasets under the much more challenging HeaRT evaluation setting. This is as opposed to other LP-specific methods that often perform similarly to standard MPNN methods.


\begin{table*}[t]
\centering
 \caption{Results (MRR) under HeaRT. Highlighted are the results ranked \colorfirst{first}, \colorsecond{second}, and  \colorthird{third}.}
 \begin{adjustbox}{}
\begin{tabular}{c|ccccccc|c}

\toprule
 \multirow{1}{*}{\bf Models} & \multicolumn{1}{c}{\bf Cora} & \multicolumn{1}{c}{\bf  Citeseer} & \multicolumn{1}{c}{\bf  Pubmed} & \multicolumn{1}{c}{\bf ogbl-collab} &\multicolumn{1}{c}{\bf ogbl-ddi}  &\multicolumn{1}{c}{\bf ogbl-ppa} &    \multicolumn{1}{c}{\bf ogbl-citation2} & {\bf Mean Rank}  \\ 
 
\midrule

\textbf{CN} & {9.78}  & {8.42}  & {2.28}  & {4.20} &  6.71 & 25.70  & 17.11 & 11.1 \\
\textbf{AA} &{11.91}  & {10.82}   &{2.63}  & {5.07} &{6.97} &{26.85} &{17.83} & 9.6  \\
\textbf{RA} &{11.81} & {10.84} & {2.47}  & \colorthird{6.29}  &{8.70}    &{28.34}  & 17.79 & 8.1 \\
\midrule
\textbf{GCN} & \colorthird{16.61 $\pm$ 0.30}     & 21.09 $\pm$ 0.88           & 7.13 $\pm$ 0.27         & {6.09 $\pm$ 0.38} & \colorfirst{13.46 $\pm$ 0.34} & 26.94 $\pm$ 0.48 & 19.98 $\pm$ 0.35 & \colorthird{4.7} \\
\textbf{SAGE}  & 14.74 $\pm$ 0.69        & 21.09 $\pm$ 1.15     & \colorsecond{9.40 $\pm$ 0.70}   & 5.53 $\pm$ 0.5 & 12.60 $\pm$ 0.72 & 27.27 $\pm$ 0.30  & \colorthird{22.05 $\pm$ 0.12} & \colorthird{4.7} \\
\textbf{GAE}  & \colorfirst{18.32 $\pm$ 0.41}   & \colorthird{25.25 $\pm$ 0.82}             & 5.27 $\pm$ 0.25             & {OOM} &{3.49 $\pm$ 1.73} & OOM & OOM  & NA     \\ \midrule
\textbf{SEAL}    & 10.67 $\pm$ 3.46      & 13.16 $\pm$ 1.66             & 5.88 $\pm$ 0.53 & \colorsecond{6.43 $\pm$ 0.32} & 9.99 $\pm$ 0.90 & {29.71 $\pm$ 0.71}& {20.60 $\pm$ 1.28} & 6.4\\
\textbf{NBFNet} & 13.56 $\pm$ 0.58          & 14.29 $\pm$ 0.80                      & >24h                       & OOM  & >24h      & OOM        & OOM        & NA \\
\midrule
\textbf{BUDDY} & 13.71 $\pm$ 0.59     & 22.84 $\pm$ 0.36        & 7.56 $\pm$ 0.18    & 5.67 $\pm$ 0.36 & 12.43 $\pm$ 0.50& 27.70 $\pm$ 0.33 & 19.17 $\pm$ 0.20 & 5.9 \\
\textbf{Neo-GNN}  & 13.95 $\pm$ 0.39     & 17.34 $\pm$ 0.84    & {7.74 $\pm$ 0.30}   & 5.23 $\pm$ 0.9 & 10.86 $\pm$ 2.16 & 21.68 $\pm$ 1.14 & 16.12 $\pm$ 0.25 & 7.4    \\
\textbf{NCN}  & 14.66 $\pm$ 0.95           & \colorfirst{28.65 $\pm$ 1.21}                   & 5.84 $\pm$ 0.22            & 5.09 $\pm$ 0.38 & \colorthird{12.86 $\pm$ 0.78}  & \colorsecond{35.06 $\pm$ 0.26}  & \colorsecond{23.35 $\pm$ 0.28} & \colorsecond{4.4}     \\
\textbf{NCNC}    & 14.98 $\pm$ 1.00  & {24.10 $\pm$ 0.65}                & \colorthird{8.58 $\pm$ 0.59}    & 4.73 $\pm$ 0.86 & >24h    & \colorthird{33.52 $\pm$ 0.26}  & 19.61 $\pm$ 0.54   & 4.8    \\ 
\midrule 
\textbf{LPFormer} & \colorsecond{16.80 $\pm$ 0.52 } & \colorsecond{26.34 $\pm$ 0.67} & \colorfirst{9.99 $\pm$ 0.52} & \colorfirst{7.62 $\pm$ 0.26} & \colorsecond{13.20 $\pm$ 0.54}& \colorfirst{40.25 $\pm$ 0.24} &   \colorfirst{24.70 $\pm$ 0.55} & \colorfirst{1.4}
\\
     
 \bottomrule
\end{tabular}
 \label{tab:heart_full}
 \end{adjustbox}
\end{table*}

\subsection{Runtime Analysis} \label{sec:runtime}

In this section, we compare the runtime of LPFormer against NCNC, which is the strongest performing baseline. 
The results are shown in Figure~\ref{fig:runtime} on all four OGB datasets
We further include the mean degree of each dataset in parentheses. We observe that LPFormer shines on denser datasets, taking significantly less time to train one epoch. This is despite that LPFormer can attend to nodes beyond the 1-hop radius of the target link. This underscores the importance of the PPR thresholding technique introduced in Section~\ref{sec:nodes_attending}, as it allows for efficient attention to a wider variety of nodes. Lastly, we note that LPFormer struggles on the ogbl-citation2 dataset due to the large number of nodes in the dataset (i.e., 2,927,963), which requires the sparse PPR matrix to be quite large. For future work we plan on exploring pre-computing the necessary PPR scores as an efficient pre-processing step, thereby removing the need to store the costly PPR matrix. Please see Appendix~\ref{sec:app_ppr_new} for more details.

\begin{figure}[H] 
    \centering      
       \includegraphics[width=0.85\linewidth]{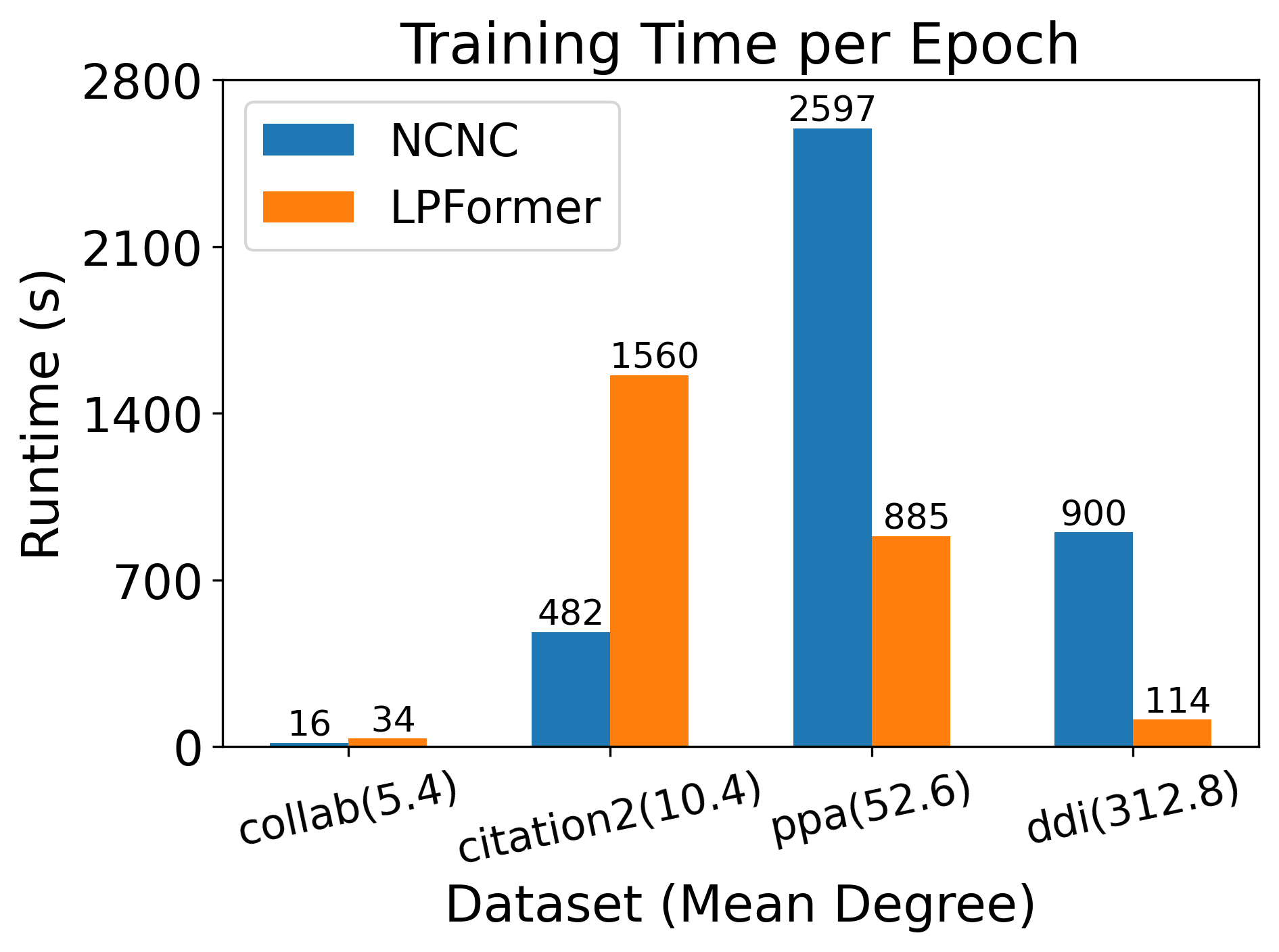}
    \caption{Comparison of training time of 1 epoch between LPFormer and NCNC. The mean degree is in parentheses.}
    \label{fig:runtime}
\end{figure}

\section{Conclusion}

In this paper we introduce a new framework, LPFormer, that aims to integrate a wider variety of pairwise information for link prediction. LPFormer does this via a specially designed graph transformer, which adaptively considers how a node pair relate to each other in the context of the graph. Extensive experiments demonstrate that LPFormer can achieve SOTA performance on a wide variety of benchmark datasets while retaining efficiency. We further demonstrate LPFormer's supremacy at modeling multiple types of LP factors. For future work, we plan on exploring other methods of incorporating multiple LP factors with an emphasis on global structural information. We also plan to investigate the potential of alternative relative positional encodings.

\begin{acks}
This research is supported by the National Science Foundation (NSF) under grant numbers CNS 2246050, IIS1845081, IIS2212032, IIS2212144, IOS2107215, DUE 2234015, DRL 2025244 and IOS2035472, the Army Research Office (ARO) under grant number W911NF-21-1-0198, the Home Depot, Cisco Systems Inc, Amazon Faculty Award, Johnson\&Johnson, JP Morgan Faculty Award and SNAP.
\end{acks}

\bibliographystyle{ACM-Reference-Format}
\bibliography{ref}


\begin{thebibliography}{59}


\ifx \showCODEN    \undefined \def \showCODEN     #1{\unskip}     \fi
\ifx \showDOI      \undefined \def \showDOI       #1{#1}\fi
\ifx \showISBNx    \undefined \def \showISBNx     #1{\unskip}     \fi
\ifx \showISBNxiii \undefined \def \showISBNxiii  #1{\unskip}     \fi
\ifx \showISSN     \undefined \def \showISSN      #1{\unskip}     \fi
\ifx \showLCCN     \undefined \def \showLCCN      #1{\unskip}     \fi
\ifx \shownote     \undefined \def \shownote      #1{#1}          \fi
\ifx \showarticletitle \undefined \def \showarticletitle #1{#1}   \fi
\ifx \showURL      \undefined \def \showURL       {\relax}        \fi
\providecommand\bibfield[2]{#2}
\providecommand\bibinfo[2]{#2}
\providecommand\natexlab[1]{#1}
\providecommand\showeprint[2][]{arXiv:#2}

\bibitem[Abbas et~al\mbox{.}(2021)]%
        {abbas2021application}
\bibfield{author}{\bibinfo{person}{Khushnood Abbas}, \bibinfo{person}{Alireza Abbasi}, \bibinfo{person}{Shi Dong}, \bibinfo{person}{Ling Niu}, \bibinfo{person}{Laihang Yu}, \bibinfo{person}{Bolun Chen}, \bibinfo{person}{Shi-Min Cai}, {and} \bibinfo{person}{Qambar Hasan}.} \bibinfo{year}{2021}\natexlab{}.
\newblock \showarticletitle{Application of network link prediction in drug discovery}.
\newblock \bibinfo{journal}{\emph{BMC bioinformatics}}  \bibinfo{volume}{22} (\bibinfo{year}{2021}), \bibinfo{pages}{1--21}.
\newblock


\bibitem[Adamic and Adar(2003)]%
        {adamic2003friends}
\bibfield{author}{\bibinfo{person}{Lada~A Adamic} {and} \bibinfo{person}{Eytan Adar}.} \bibinfo{year}{2003}\natexlab{}.
\newblock \showarticletitle{Friends and neighbors on the web}.
\newblock \bibinfo{journal}{\emph{Social networks}} \bibinfo{volume}{25}, \bibinfo{number}{3} (\bibinfo{year}{2003}), \bibinfo{pages}{211--230}.
\newblock


\bibitem[Andersen et~al\mbox{.}(2006)]%
        {andersen2006local}
\bibfield{author}{\bibinfo{person}{Reid Andersen}, \bibinfo{person}{Fan Chung}, {and} \bibinfo{person}{Kevin Lang}.} \bibinfo{year}{2006}\natexlab{}.
\newblock \showarticletitle{Local graph partitioning using pagerank vectors}. In \bibinfo{booktitle}{\emph{2006 47th Annual IEEE Symposium on Foundations of Computer Science (FOCS'06)}}. IEEE, \bibinfo{pages}{475--486}.
\newblock


\bibitem[Bahdanau et~al\mbox{.}(2015)]%
        {bahdanau2015neural}
\bibfield{author}{\bibinfo{person}{Dzmitry Bahdanau}, \bibinfo{person}{Kyung~Hyun Cho}, {and} \bibinfo{person}{Yoshua Bengio}.} \bibinfo{year}{2015}\natexlab{}.
\newblock \showarticletitle{Neural machine translation by jointly learning to align and translate}. In \bibinfo{booktitle}{\emph{3rd International Conference on Learning Representations, ICLR 2015}}.
\newblock


\bibitem[Barab{\^a}si et~al\mbox{.}(2002)]%
        {barabasi2002evolution}
\bibfield{author}{\bibinfo{person}{Albert-Laszlo Barab{\^a}si}, \bibinfo{person}{Hawoong Jeong}, \bibinfo{person}{Zoltan N{\'e}da}, \bibinfo{person}{Erzsebet Ravasz}, \bibinfo{person}{Andras Schubert}, {and} \bibinfo{person}{Tamas Vicsek}.} \bibinfo{year}{2002}\natexlab{}.
\newblock \showarticletitle{Evolution of the social network of scientific collaborations}.
\newblock \bibinfo{journal}{\emph{Physica A: Statistical mechanics and its applications}} \bibinfo{volume}{311}, \bibinfo{number}{3-4} (\bibinfo{year}{2002}), \bibinfo{pages}{590--614}.
\newblock


\bibitem[Bojchevski et~al\mbox{.}(2020)]%
        {bojchevski2020scaling}
\bibfield{author}{\bibinfo{person}{Aleksandar Bojchevski}, \bibinfo{person}{Johannes Gasteiger}, \bibinfo{person}{Bryan Perozzi}, \bibinfo{person}{Amol Kapoor}, \bibinfo{person}{Martin Blais}, \bibinfo{person}{Benedek R{\'o}zemberczki}, \bibinfo{person}{Michal Lukasik}, {and} \bibinfo{person}{Stephan G{\"u}nnemann}.} \bibinfo{year}{2020}\natexlab{}.
\newblock \showarticletitle{Scaling graph neural networks with approximate pagerank}. In \bibinfo{booktitle}{\emph{Proceedings of the 26th ACM SIGKDD International Conference on Knowledge Discovery \& Data Mining}}. \bibinfo{pages}{2464--2473}.
\newblock


\bibitem[Brin and Page(1998)]%
        {pagerank}
\bibfield{author}{\bibinfo{person}{Sergey Brin} {and} \bibinfo{person}{Lawrence Page}.} \bibinfo{year}{1998}\natexlab{}.
\newblock \showarticletitle{The anatomy of a large-scale hypertextual web search engine}.
\newblock \bibinfo{journal}{\emph{Computer networks and ISDN systems}} \bibinfo{volume}{30}, \bibinfo{number}{1-7} (\bibinfo{year}{1998}), \bibinfo{pages}{107--117}.
\newblock


\bibitem[Broder(1997)]%
        {broder1997resemblance}
\bibfield{author}{\bibinfo{person}{Andrei~Z Broder}.} \bibinfo{year}{1997}\natexlab{}.
\newblock \showarticletitle{On the resemblance and containment of documents}. In \bibinfo{booktitle}{\emph{Proceedings. Compression and Complexity of SEQUENCES 1997 (Cat. No. 97TB100171)}}. IEEE, \bibinfo{pages}{21--29}.
\newblock


\bibitem[Brody et~al\mbox{.}(2022)]%
        {gatv2}
\bibfield{author}{\bibinfo{person}{Shaked Brody}, \bibinfo{person}{Uri Alon}, {and} \bibinfo{person}{Eran Yahav}.} \bibinfo{year}{2022}\natexlab{}.
\newblock \showarticletitle{How Attentive are Graph Attention Networks?}. In \bibinfo{booktitle}{\emph{International Conference on Learning Representations}}.
\newblock
\urldef\tempurl%
\url{https://openreview.net/forum?id=F72ximsx7C1}
\showURL{%
\tempurl}


\bibitem[Chamberlain et~al\mbox{.}(2022)]%
        {chamberlain2022graph}
\bibfield{author}{\bibinfo{person}{Benjamin~Paul Chamberlain}, \bibinfo{person}{Sergey Shirobokov}, \bibinfo{person}{Emanuele Rossi}, \bibinfo{person}{Fabrizio Frasca}, \bibinfo{person}{Thomas Markovich}, \bibinfo{person}{Nils Hammerla}, \bibinfo{person}{Michael~M Bronstein}, {and} \bibinfo{person}{Max Hansmire}.} \bibinfo{year}{2022}\natexlab{}.
\newblock \showarticletitle{Graph Neural Networks for Link Prediction with Subgraph Sketching}.
\newblock \bibinfo{journal}{\emph{arXiv preprint arXiv:2209.15486}} (\bibinfo{year}{2022}).
\newblock


\bibitem[Chen et~al\mbox{.}(2022)]%
        {chen2022nagphormer}
\bibfield{author}{\bibinfo{person}{Jinsong Chen}, \bibinfo{person}{Kaiyuan Gao}, \bibinfo{person}{Gaichao Li}, {and} \bibinfo{person}{Kun He}.} \bibinfo{year}{2022}\natexlab{}.
\newblock \showarticletitle{NAGphormer: A tokenized graph transformer for node classification in large graphs}. In \bibinfo{booktitle}{\emph{The Eleventh International Conference on Learning Representations}}.
\newblock


\bibitem[Chen et~al\mbox{.}(2021)]%
        {chen2021hitter}
\bibfield{author}{\bibinfo{person}{Sanxing Chen}, \bibinfo{person}{Xiaodong Liu}, \bibinfo{person}{Jianfeng Gao}, \bibinfo{person}{Jian Jiao}, \bibinfo{person}{Ruofei Zhang}, {and} \bibinfo{person}{Yangfeng Ji}.} \bibinfo{year}{2021}\natexlab{}.
\newblock \showarticletitle{HittER: Hierarchical Transformers for Knowledge Graph Embeddings}. In \bibinfo{booktitle}{\emph{Proceedings of the 2021 Conference on Empirical Methods in Natural Language Processing}}. \bibinfo{pages}{10395--10407}.
\newblock


\bibitem[Chung(2007)]%
        {chung2007heat}
\bibfield{author}{\bibinfo{person}{Fan Chung}.} \bibinfo{year}{2007}\natexlab{}.
\newblock \showarticletitle{The heat kernel as the pagerank of a graph}.
\newblock \bibinfo{journal}{\emph{Proceedings of the National Academy of Sciences}} \bibinfo{volume}{104}, \bibinfo{number}{50} (\bibinfo{year}{2007}), \bibinfo{pages}{19735--19740}.
\newblock


\bibitem[Correia et~al\mbox{.}(2019)]%
        {correia2019adaptively}
\bibfield{author}{\bibinfo{person}{Gon{\c{c}}alo~M Correia}, \bibinfo{person}{Vlad Niculae}, {and} \bibinfo{person}{Andr{\'e}~FT Martins}.} \bibinfo{year}{2019}\natexlab{}.
\newblock \showarticletitle{Adaptively Sparse Transformers}. In \bibinfo{booktitle}{\emph{Proceedings of the 2019 Conference on Empirical Methods in Natural Language Processing and the 9th International Joint Conference on Natural Language Processing (EMNLP-IJCNLP)}}. \bibinfo{pages}{2174--2184}.
\newblock


\bibitem[Daud et~al\mbox{.}(2020)]%
        {daud2020applications}
\bibfield{author}{\bibinfo{person}{Nur~Nasuha Daud}, \bibinfo{person}{Siti~Hafizah Ab~Hamid}, \bibinfo{person}{Muntadher Saadoon}, \bibinfo{person}{Firdaus Sahran}, {and} \bibinfo{person}{Nor~Badrul Anuar}.} \bibinfo{year}{2020}\natexlab{}.
\newblock \showarticletitle{Applications of link prediction in social networks: A review}.
\newblock \bibinfo{journal}{\emph{Journal of Network and Computer Applications}}  \bibinfo{volume}{166} (\bibinfo{year}{2020}), \bibinfo{pages}{102716}.
\newblock


\bibitem[Dong et~al\mbox{.}(2017)]%
        {dong2017structural}
\bibfield{author}{\bibinfo{person}{Yuxiao Dong}, \bibinfo{person}{Reid~A Johnson}, \bibinfo{person}{Jian Xu}, {and} \bibinfo{person}{Nitesh~V Chawla}.} \bibinfo{year}{2017}\natexlab{}.
\newblock \showarticletitle{Structural diversity and homophily: A study across more than one hundred big networks}. In \bibinfo{booktitle}{\emph{Proceedings of the 23rd ACM SIGKDD International Conference on Knowledge Discovery and Data Mining}}. \bibinfo{pages}{807--816}.
\newblock


\bibitem[Flajolet et~al\mbox{.}(2007)]%
        {flajolet2007hyperloglog}
\bibfield{author}{\bibinfo{person}{Philippe Flajolet}, \bibinfo{person}{{\'E}ric Fusy}, \bibinfo{person}{Olivier Gandouet}, {and} \bibinfo{person}{Fr{\'e}d{\'e}ric Meunier}.} \bibinfo{year}{2007}\natexlab{}.
\newblock \showarticletitle{Hyperloglog: the analysis of a near-optimal cardinality estimation algorithm}.
\newblock \bibinfo{journal}{\emph{Discrete mathematics \& theoretical computer science}} \bibinfo{number}{Proceedings} (\bibinfo{year}{2007}).
\newblock


\bibitem[Gilmer et~al\mbox{.}(2017)]%
        {gilmer2017neural}
\bibfield{author}{\bibinfo{person}{Justin Gilmer}, \bibinfo{person}{Samuel~S Schoenholz}, \bibinfo{person}{Patrick~F Riley}, \bibinfo{person}{Oriol Vinyals}, {and} \bibinfo{person}{George~E Dahl}.} \bibinfo{year}{2017}\natexlab{}.
\newblock \showarticletitle{Neural message passing for quantum chemistry}. In \bibinfo{booktitle}{\emph{International conference on machine learning}}. PMLR, \bibinfo{pages}{1263--1272}.
\newblock


\bibitem[Gleich et~al\mbox{.}(2015)]%
        {gleich2015localization}
\bibfield{author}{\bibinfo{person}{David~F Gleich}, \bibinfo{person}{Kyle Kloster}, {and} \bibinfo{person}{Huda Nassar}.} \bibinfo{year}{2015}\natexlab{}.
\newblock \showarticletitle{Localization in seeded pagerank}.
\newblock \bibinfo{journal}{\emph{arXiv preprint arXiv:1509.00016}} (\bibinfo{year}{2015}).
\newblock


\bibitem[Hamilton et~al\mbox{.}(2017)]%
        {hamilton2017inductive}
\bibfield{author}{\bibinfo{person}{Will Hamilton}, \bibinfo{person}{Zhitao Ying}, {and} \bibinfo{person}{Jure Leskovec}.} \bibinfo{year}{2017}\natexlab{}.
\newblock \showarticletitle{Inductive representation learning on large graphs}.
\newblock \bibinfo{journal}{\emph{Advances in neural information processing systems}}  \bibinfo{volume}{30} (\bibinfo{year}{2017}).
\newblock


\bibitem[Hu et~al\mbox{.}(2020)]%
        {ogb}
\bibfield{author}{\bibinfo{person}{Weihua Hu}, \bibinfo{person}{Matthias Fey}, \bibinfo{person}{Marinka Zitnik}, \bibinfo{person}{Yuxiao Dong}, \bibinfo{person}{Hongyu Ren}, \bibinfo{person}{Bowen Liu}, \bibinfo{person}{Michele Catasta}, {and} \bibinfo{person}{Jure Leskovec}.} \bibinfo{year}{2020}\natexlab{}.
\newblock \showarticletitle{Open graph benchmark: Datasets for machine learning on graphs}.
\newblock \bibinfo{journal}{\emph{Advances in neural information processing systems}}  \bibinfo{volume}{33} (\bibinfo{year}{2020}), \bibinfo{pages}{22118--22133}.
\newblock


\bibitem[Huang et~al\mbox{.}(2015)]%
        {huang2015triadic}
\bibfield{author}{\bibinfo{person}{Hong Huang}, \bibinfo{person}{Jie Tang}, \bibinfo{person}{Lu Liu}, \bibinfo{person}{JarDer Luo}, {and} \bibinfo{person}{Xiaoming Fu}.} \bibinfo{year}{2015}\natexlab{}.
\newblock \showarticletitle{Triadic closure pattern analysis and prediction in social networks}.
\newblock \bibinfo{journal}{\emph{IEEE Transactions on Knowledge and Data Engineering}} \bibinfo{volume}{27}, \bibinfo{number}{12} (\bibinfo{year}{2015}), \bibinfo{pages}{3374--3389}.
\newblock


\bibitem[Huang et~al\mbox{.}(2005)]%
        {huang2005link}
\bibfield{author}{\bibinfo{person}{Zan Huang}, \bibinfo{person}{Xin Li}, {and} \bibinfo{person}{Hsinchun Chen}.} \bibinfo{year}{2005}\natexlab{}.
\newblock \showarticletitle{Link prediction approach to collaborative filtering}. In \bibinfo{booktitle}{\emph{Proceedings of the 5th ACM/IEEE-CS joint conference on Digital libraries}}. \bibinfo{pages}{141--142}.
\newblock


\bibitem[Katz(1953)]%
        {katz1953new}
\bibfield{author}{\bibinfo{person}{Leo Katz}.} \bibinfo{year}{1953}\natexlab{}.
\newblock \showarticletitle{A new status index derived from sociometric analysis}.
\newblock \bibinfo{journal}{\emph{Psychometrika}} \bibinfo{volume}{18}, \bibinfo{number}{1} (\bibinfo{year}{1953}), \bibinfo{pages}{39--43}.
\newblock


\bibitem[Kim et~al\mbox{.}(2022)]%
        {kim2022pure}
\bibfield{author}{\bibinfo{person}{Jinwoo Kim}, \bibinfo{person}{Dat Nguyen}, \bibinfo{person}{Seonwoo Min}, \bibinfo{person}{Sungjun Cho}, \bibinfo{person}{Moontae Lee}, \bibinfo{person}{Honglak Lee}, {and} \bibinfo{person}{Seunghoon Hong}.} \bibinfo{year}{2022}\natexlab{}.
\newblock \showarticletitle{Pure transformers are powerful graph learners}.
\newblock \bibinfo{journal}{\emph{Advances in Neural Information Processing Systems}}  \bibinfo{volume}{35} (\bibinfo{year}{2022}), \bibinfo{pages}{14582--14595}.
\newblock


\bibitem[Kipf and Welling(2016a)]%
        {kipf2016semi}
\bibfield{author}{\bibinfo{person}{Thomas~N Kipf} {and} \bibinfo{person}{Max Welling}.} \bibinfo{year}{2016}\natexlab{a}.
\newblock \showarticletitle{Semi-supervised classification with graph convolutional networks}.
\newblock \bibinfo{journal}{\emph{arXiv preprint arXiv:1609.02907}} (\bibinfo{year}{2016}).
\newblock


\bibitem[Kipf and Welling(2016b)]%
        {kipf2016variational}
\bibfield{author}{\bibinfo{person}{Thomas~N Kipf} {and} \bibinfo{person}{Max Welling}.} \bibinfo{year}{2016}\natexlab{b}.
\newblock \showarticletitle{Variational graph auto-encoders}.
\newblock \bibinfo{journal}{\emph{arXiv preprint arXiv:1611.07308}} (\bibinfo{year}{2016}).
\newblock


\bibitem[Kipf and Welling(2017)]%
        {kipf2017semi}
\bibfield{author}{\bibinfo{person}{Thomas~N. Kipf} {and} \bibinfo{person}{Max Welling}.} \bibinfo{year}{2017}\natexlab{}.
\newblock \showarticletitle{Semi-Supervised Classification with Graph Convolutional Networks}. In \bibinfo{booktitle}{\emph{International Conference on Learning Representations (ICLR)}}.
\newblock


\bibitem[Kreuzer et~al\mbox{.}(2021)]%
        {kreuzer2021rethinking}
\bibfield{author}{\bibinfo{person}{Devin Kreuzer}, \bibinfo{person}{Dominique Beaini}, \bibinfo{person}{Will Hamilton}, \bibinfo{person}{Vincent L{\'e}tourneau}, {and} \bibinfo{person}{Prudencio Tossou}.} \bibinfo{year}{2021}\natexlab{}.
\newblock \showarticletitle{Rethinking graph transformers with spectral attention}.
\newblock \bibinfo{journal}{\emph{Advances in Neural Information Processing Systems}}  \bibinfo{volume}{34} (\bibinfo{year}{2021}), \bibinfo{pages}{21618--21629}.
\newblock


\bibitem[Li et~al\mbox{.}(2023)]%
        {li2023evaluating}
\bibfield{author}{\bibinfo{person}{Juanhui Li}, \bibinfo{person}{Harry Shomer}, \bibinfo{person}{Haitao Mao}, \bibinfo{person}{Shenglai Zeng}, \bibinfo{person}{Yao Ma}, \bibinfo{person}{Neil Shah}, \bibinfo{person}{Jiliang Tang}, {and} \bibinfo{person}{Dawei Yin}.} \bibinfo{year}{2023}\natexlab{}.
\newblock \showarticletitle{Evaluating Graph Neural Networks for Link Prediction: Current Pitfalls and New Benchmarking}.
\newblock \bibinfo{journal}{\emph{arXiv preprint arXiv:2306.10453}} (\bibinfo{year}{2023}).
\newblock


\bibitem[Li et~al\mbox{.}(2020)]%
        {li2020distance}
\bibfield{author}{\bibinfo{person}{Pan Li}, \bibinfo{person}{Yanbang Wang}, \bibinfo{person}{Hongwei Wang}, {and} \bibinfo{person}{Jure Leskovec}.} \bibinfo{year}{2020}\natexlab{}.
\newblock \showarticletitle{Distance encoding: Design provably more powerful neural networks for graph representation learning}.
\newblock \bibinfo{journal}{\emph{Advances in Neural Information Processing Systems}}  \bibinfo{volume}{33} (\bibinfo{year}{2020}), \bibinfo{pages}{4465--4478}.
\newblock


\bibitem[Liben-Nowell and Kleinberg(2003)]%
        {liben2003link}
\bibfield{author}{\bibinfo{person}{David Liben-Nowell} {and} \bibinfo{person}{Jon Kleinberg}.} \bibinfo{year}{2003}\natexlab{}.
\newblock \showarticletitle{The link prediction problem for social networks}. In \bibinfo{booktitle}{\emph{Proceedings of the twelfth international conference on Information and knowledge management}}. \bibinfo{pages}{556--559}.
\newblock


\bibitem[Mao et~al\mbox{.}(2024)]%
        {mao2024demystifying}
\bibfield{author}{\bibinfo{person}{Haitao Mao}, \bibinfo{person}{Zhikai Chen}, \bibinfo{person}{Wei Jin}, \bibinfo{person}{Haoyu Han}, \bibinfo{person}{Yao Ma}, \bibinfo{person}{Tong Zhao}, \bibinfo{person}{Neil Shah}, {and} \bibinfo{person}{Jiliang Tang}.} \bibinfo{year}{2024}\natexlab{}.
\newblock \showarticletitle{Demystifying Structural Disparity in Graph Neural Networks: Can One Size Fit All?}
\newblock \bibinfo{journal}{\emph{Advances in Neural Information Processing Systems}}  \bibinfo{volume}{36} (\bibinfo{year}{2024}).
\newblock


\bibitem[Mao et~al\mbox{.}(2023)]%
        {mao2023revisiting}
\bibfield{author}{\bibinfo{person}{Haitao Mao}, \bibinfo{person}{Juanhui Li}, \bibinfo{person}{Harry Shomer}, \bibinfo{person}{Bingheng Li}, \bibinfo{person}{Wenqi Fan}, \bibinfo{person}{Yao Ma}, \bibinfo{person}{Tong Zhao}, \bibinfo{person}{Neil Shah}, {and} \bibinfo{person}{Jiliang Tang}.} \bibinfo{year}{2023}\natexlab{}.
\newblock \bibinfo{title}{Revisiting Link Prediction: A Data Perspective}.
\newblock
\newblock
\showeprint[arxiv]{2310.00793}~[cs.SI]


\bibitem[Maruf et~al\mbox{.}(2019)]%
        {maruf2019selective}
\bibfield{author}{\bibinfo{person}{Sameen Maruf}, \bibinfo{person}{Andr{\'e}~FT Martins}, {and} \bibinfo{person}{Gholamreza Haffari}.} \bibinfo{year}{2019}\natexlab{}.
\newblock \showarticletitle{Selective Attention for Context-aware Neural Machine Translation}. In \bibinfo{booktitle}{\emph{Proceedings of NAACL-HLT}}. \bibinfo{pages}{3092--3102}.
\newblock


\bibitem[Mialon et~al\mbox{.}(2021)]%
        {mialon2021graphit}
\bibfield{author}{\bibinfo{person}{Gr{\'e}goire Mialon}, \bibinfo{person}{Dexiong Chen}, \bibinfo{person}{Margot Selosse}, {and} \bibinfo{person}{Julien Mairal}.} \bibinfo{year}{2021}\natexlab{}.
\newblock \showarticletitle{Graphit: Encoding graph structure in transformers}.
\newblock \bibinfo{journal}{\emph{arXiv preprint arXiv:2106.05667}} (\bibinfo{year}{2021}).
\newblock


\bibitem[M{\"u}ller et~al\mbox{.}(2023)]%
        {muller2023attending}
\bibfield{author}{\bibinfo{person}{Luis M{\"u}ller}, \bibinfo{person}{Mikhail Galkin}, \bibinfo{person}{Christopher Morris}, {and} \bibinfo{person}{Ladislav Ramp{\'a}{\v{s}}ek}.} \bibinfo{year}{2023}\natexlab{}.
\newblock \showarticletitle{Attending to graph transformers}.
\newblock \bibinfo{journal}{\emph{arXiv preprint arXiv:2302.04181}} (\bibinfo{year}{2023}).
\newblock


\bibitem[Murase et~al\mbox{.}(2019)]%
        {murase2019structural}
\bibfield{author}{\bibinfo{person}{Yohsuke Murase}, \bibinfo{person}{Hang-Hyun Jo}, \bibinfo{person}{J{\'a}nos T{\"o}r{\"o}k}, \bibinfo{person}{J{\'a}nos Kert{\'e}sz}, {and} \bibinfo{person}{Kimmo Kaski}.} \bibinfo{year}{2019}\natexlab{}.
\newblock \showarticletitle{Structural transition in social networks: The role of homophily}.
\newblock \bibinfo{journal}{\emph{Scientific reports}} \bibinfo{volume}{9}, \bibinfo{number}{1} (\bibinfo{year}{2019}), \bibinfo{pages}{4310}.
\newblock


\bibitem[Nassar et~al\mbox{.}(2015)]%
        {nassar2015strong}
\bibfield{author}{\bibinfo{person}{Huda Nassar}, \bibinfo{person}{Kyle Kloster}, {and} \bibinfo{person}{David~F Gleich}.} \bibinfo{year}{2015}\natexlab{}.
\newblock \showarticletitle{Strong Localization in Personalized PageRank Vectors}. In \bibinfo{booktitle}{\emph{Proceedings of the 12th International Workshop on Algorithms and Models for the Web Graph-Volume 9479}}. \bibinfo{pages}{190--202}.
\newblock


\bibitem[Newman(2001)]%
        {newman2001clustering}
\bibfield{author}{\bibinfo{person}{Mark~EJ Newman}.} \bibinfo{year}{2001}\natexlab{}.
\newblock \showarticletitle{Clustering and preferential attachment in growing networks}.
\newblock \bibinfo{journal}{\emph{Physical review E}} \bibinfo{volume}{64}, \bibinfo{number}{2} (\bibinfo{year}{2001}), \bibinfo{pages}{025102}.
\newblock


\bibitem[Nickel et~al\mbox{.}(2014)]%
        {nickel2014reducing}
\bibfield{author}{\bibinfo{person}{Maximilian Nickel}, \bibinfo{person}{Xueyan Jiang}, {and} \bibinfo{person}{Volker Tresp}.} \bibinfo{year}{2014}\natexlab{}.
\newblock \showarticletitle{Reducing the rank in relational factorization models by including observable patterns}.
\newblock \bibinfo{journal}{\emph{Advances in Neural Information Processing Systems}}  \bibinfo{volume}{27} (\bibinfo{year}{2014}).
\newblock


\bibitem[Pahuja et~al\mbox{.}(2023)]%
        {pahuja2023retrieve}
\bibfield{author}{\bibinfo{person}{Vardaan Pahuja}, \bibinfo{person}{Boshi Wang}, \bibinfo{person}{Hugo Latapie}, \bibinfo{person}{Jayanth Srinivasa}, {and} \bibinfo{person}{Yu Su}.} \bibinfo{year}{2023}\natexlab{}.
\newblock \showarticletitle{A retrieve-and-read framework for knowledge graph link prediction}. In \bibinfo{booktitle}{\emph{Proceedings of the 32nd ACM International Conference on Information and Knowledge Management}}. \bibinfo{pages}{1992--2002}.
\newblock


\bibitem[Ramp{\'a}{\v{s}}ek et~al\mbox{.}(2022)]%
        {gps}
\bibfield{author}{\bibinfo{person}{Ladislav Ramp{\'a}{\v{s}}ek}, \bibinfo{person}{Michael Galkin}, \bibinfo{person}{Vijay~Prakash Dwivedi}, \bibinfo{person}{Anh~Tuan Luu}, \bibinfo{person}{Guy Wolf}, {and} \bibinfo{person}{Dominique Beaini}.} \bibinfo{year}{2022}\natexlab{}.
\newblock \showarticletitle{Recipe for a general, powerful, scalable graph transformer}.
\newblock \bibinfo{journal}{\emph{Advances in Neural Information Processing Systems}}  \bibinfo{volume}{35} (\bibinfo{year}{2022}), \bibinfo{pages}{14501--14515}.
\newblock


\bibitem[Rozemberczki et~al\mbox{.}(2021)]%
        {heterophilic_datasets}
\bibfield{author}{\bibinfo{person}{Benedek Rozemberczki}, \bibinfo{person}{Carl Allen}, {and} \bibinfo{person}{Rik Sarkar}.} \bibinfo{year}{2021}\natexlab{}.
\newblock \showarticletitle{Multi-scale attributed node embedding}.
\newblock \bibinfo{journal}{\emph{Journal of Complex Networks}} \bibinfo{volume}{9}, \bibinfo{number}{2} (\bibinfo{year}{2021}), \bibinfo{pages}{cnab014}.
\newblock


\bibitem[Srinivasan and Ribeiro(2019)]%
        {srinivasan2019equivalence}
\bibfield{author}{\bibinfo{person}{Balasubramaniam Srinivasan} {and} \bibinfo{person}{Bruno Ribeiro}.} \bibinfo{year}{2019}\natexlab{}.
\newblock \showarticletitle{On the Equivalence between Positional Node Embeddings and Structural Graph Representations}. In \bibinfo{booktitle}{\emph{International Conference on Learning Representations}}.
\newblock


\bibitem[Vaswani et~al\mbox{.}(2017)]%
        {vaswani2017attention}
\bibfield{author}{\bibinfo{person}{Ashish Vaswani}, \bibinfo{person}{Noam Shazeer}, \bibinfo{person}{Niki Parmar}, \bibinfo{person}{Jakob Uszkoreit}, \bibinfo{person}{Llion Jones}, \bibinfo{person}{Aidan~N Gomez}, \bibinfo{person}{{\L}ukasz Kaiser}, {and} \bibinfo{person}{Illia Polosukhin}.} \bibinfo{year}{2017}\natexlab{}.
\newblock \showarticletitle{Attention is all you need}.
\newblock \bibinfo{journal}{\emph{Advances in neural information processing systems}}  \bibinfo{volume}{30} (\bibinfo{year}{2017}).
\newblock


\bibitem[Velickovic et~al\mbox{.}(2017)]%
        {velickovic2017graph}
\bibfield{author}{\bibinfo{person}{Petar Velickovic}, \bibinfo{person}{Guillem Cucurull}, \bibinfo{person}{Arantxa Casanova}, \bibinfo{person}{Adriana Romero}, \bibinfo{person}{Pietro Lio}, {and} \bibinfo{person}{Yoshua Bengio}.} \bibinfo{year}{2017}\natexlab{}.
\newblock \showarticletitle{Graph attention networks}.
\newblock \bibinfo{journal}{\emph{stat}}  \bibinfo{volume}{1050} (\bibinfo{year}{2017}), \bibinfo{pages}{20}.
\newblock


\bibitem[Wang et~al\mbox{.}(2023)]%
        {ncn}
\bibfield{author}{\bibinfo{person}{Xiyuan Wang}, \bibinfo{person}{Haotong Yang}, {and} \bibinfo{person}{Muhan Zhang}.} \bibinfo{year}{2023}\natexlab{}.
\newblock \showarticletitle{Neural Common Neighbor with Completion for Link Prediction}.
\newblock \bibinfo{journal}{\emph{arXiv preprint arXiv:2302.00890}} (\bibinfo{year}{2023}).
\newblock


\bibitem[Wu et~al\mbox{.}(2022)]%
        {wu2022nodeformer}
\bibfield{author}{\bibinfo{person}{Qitian Wu}, \bibinfo{person}{Wentao Zhao}, \bibinfo{person}{Zenan Li}, \bibinfo{person}{David~P Wipf}, {and} \bibinfo{person}{Junchi Yan}.} \bibinfo{year}{2022}\natexlab{}.
\newblock \showarticletitle{Nodeformer: A scalable graph structure learning transformer for node classification}.
\newblock \bibinfo{journal}{\emph{Advances in Neural Information Processing Systems}}  \bibinfo{volume}{35} (\bibinfo{year}{2022}), \bibinfo{pages}{27387--27401}.
\newblock


\bibitem[Yang et~al\mbox{.}(2016)]%
        {planetoid}
\bibfield{author}{\bibinfo{person}{Zhilin Yang}, \bibinfo{person}{William Cohen}, {and} \bibinfo{person}{Ruslan Salakhudinov}.} \bibinfo{year}{2016}\natexlab{}.
\newblock \showarticletitle{Revisiting semi-supervised learning with graph embeddings}. In \bibinfo{booktitle}{\emph{International conference on machine learning}}. PMLR, \bibinfo{pages}{40--48}.
\newblock


\bibitem[Ying et~al\mbox{.}(2021)]%
        {graphormer}
\bibfield{author}{\bibinfo{person}{Chengxuan Ying}, \bibinfo{person}{Tianle Cai}, \bibinfo{person}{Shengjie Luo}, \bibinfo{person}{Shuxin Zheng}, \bibinfo{person}{Guolin Ke}, \bibinfo{person}{Di He}, \bibinfo{person}{Yanming Shen}, {and} \bibinfo{person}{Tie-Yan Liu}.} \bibinfo{year}{2021}\natexlab{}.
\newblock \showarticletitle{Do transformers really perform badly for graph representation?}
\newblock \bibinfo{journal}{\emph{Advances in Neural Information Processing Systems}}  \bibinfo{volume}{34} (\bibinfo{year}{2021}), \bibinfo{pages}{28877--28888}.
\newblock


\bibitem[Ying et~al\mbox{.}(2018)]%
        {ying2018graph}
\bibfield{author}{\bibinfo{person}{Rex Ying}, \bibinfo{person}{Ruining He}, \bibinfo{person}{Kaifeng Chen}, \bibinfo{person}{Pong Eksombatchai}, \bibinfo{person}{William~L Hamilton}, {and} \bibinfo{person}{Jure Leskovec}.} \bibinfo{year}{2018}\natexlab{}.
\newblock \showarticletitle{Graph convolutional neural networks for web-scale recommender systems}. In \bibinfo{booktitle}{\emph{Proceedings of the 24th ACM SIGKDD international conference on knowledge discovery \& data mining}}. \bibinfo{pages}{974--983}.
\newblock


\bibitem[Yun et~al\mbox{.}(2021)]%
        {yun2021neo}
\bibfield{author}{\bibinfo{person}{Seongjun Yun}, \bibinfo{person}{Seoyoon Kim}, \bibinfo{person}{Junhyun Lee}, \bibinfo{person}{Jaewoo Kang}, {and} \bibinfo{person}{Hyunwoo~J Kim}.} \bibinfo{year}{2021}\natexlab{}.
\newblock \showarticletitle{Neo-gnns: Neighborhood overlap-aware graph neural networks for link prediction}.
\newblock \bibinfo{journal}{\emph{Advances in Neural Information Processing Systems}}  \bibinfo{volume}{34} (\bibinfo{year}{2021}), \bibinfo{pages}{13683--13694}.
\newblock


\bibitem[Zhang and Chen(2018)]%
        {seal}
\bibfield{author}{\bibinfo{person}{Muhan Zhang} {and} \bibinfo{person}{Yixin Chen}.} \bibinfo{year}{2018}\natexlab{}.
\newblock \showarticletitle{Link prediction based on graph neural networks}.
\newblock \bibinfo{journal}{\emph{Advances in neural information processing systems}}  \bibinfo{volume}{31} (\bibinfo{year}{2018}).
\newblock


\bibitem[Zhang et~al\mbox{.}(2021a)]%
        {labeling_trick}
\bibfield{author}{\bibinfo{person}{Muhan Zhang}, \bibinfo{person}{Pan Li}, \bibinfo{person}{Yinglong Xia}, \bibinfo{person}{Kai Wang}, {and} \bibinfo{person}{Long Jin}.} \bibinfo{year}{2021}\natexlab{a}.
\newblock \showarticletitle{Labeling trick: A theory of using graph neural networks for multi-node representation learning}.
\newblock \bibinfo{journal}{\emph{Advances in Neural Information Processing Systems}}  \bibinfo{volume}{34} (\bibinfo{year}{2021}), \bibinfo{pages}{9061--9073}.
\newblock


\bibitem[Zhang et~al\mbox{.}(2021b)]%
        {zhang2021labeling}
\bibfield{author}{\bibinfo{person}{Muhan Zhang}, \bibinfo{person}{Pan Li}, \bibinfo{person}{Yinglong Xia}, \bibinfo{person}{Kai Wang}, {and} \bibinfo{person}{Long Jin}.} \bibinfo{year}{2021}\natexlab{b}.
\newblock \showarticletitle{Labeling trick: A theory of using graph neural networks for multi-node representation learning}.
\newblock \bibinfo{journal}{\emph{Advances in Neural Information Processing Systems}}  \bibinfo{volume}{34} (\bibinfo{year}{2021}), \bibinfo{pages}{9061--9073}.
\newblock


\bibitem[Zhao et~al\mbox{.}(2017)]%
        {zhao2017leveraging}
\bibfield{author}{\bibinfo{person}{He Zhao}, \bibinfo{person}{Lan Du}, {and} \bibinfo{person}{Wray Buntine}.} \bibinfo{year}{2017}\natexlab{}.
\newblock \showarticletitle{Leveraging node attributes for incomplete relational data}. In \bibinfo{booktitle}{\emph{International conference on machine learning}}. PMLR, \bibinfo{pages}{4072--4081}.
\newblock


\bibitem[Zhou et~al\mbox{.}(2009)]%
        {zhou2009predicting}
\bibfield{author}{\bibinfo{person}{Tao Zhou}, \bibinfo{person}{Linyuan L{\"u}}, {and} \bibinfo{person}{Yi-Cheng Zhang}.} \bibinfo{year}{2009}\natexlab{}.
\newblock \showarticletitle{Predicting missing links via local information}.
\newblock \bibinfo{journal}{\emph{The European Physical Journal B}}  \bibinfo{volume}{71} (\bibinfo{year}{2009}), \bibinfo{pages}{623--630}.
\newblock


\bibitem[Zhu et~al\mbox{.}(2021)]%
        {nbfnet}
\bibfield{author}{\bibinfo{person}{Zhaocheng Zhu}, \bibinfo{person}{Zuobai Zhang}, \bibinfo{person}{Louis-Pascal Xhonneux}, {and} \bibinfo{person}{Jian Tang}.} \bibinfo{year}{2021}\natexlab{}.
\newblock \showarticletitle{Neural bellman-ford networks: A general graph neural network framework for link prediction}.
\newblock \bibinfo{journal}{\emph{Advances in Neural Information Processing Systems}}  \bibinfo{volume}{34} (\bibinfo{year}{2021}), \bibinfo{pages}{29476--29490}.
\newblock


\end{thebibliography}

\onecolumn
\appendix

\section{Existing Formulations of Pairwise Encodings} \label{sec:app_pair_encoding}

In this section we give an overview of existing formulations of pairwise encodings using in DP-MPNNs. The standard formulation of DP-MPNNs is given in Eq.~\ref{eq:gnn_pairwise} where $s(a, b)$ is the pairwise encoding. We briefly describe other existing solutions below:
\vskip 1em
\noindent {\bf NCN}~\cite{ncn}: NCN only considers the CNs of the target link $(a, b)$ by summing the node representation of each. The pairwise encoding, $s(a, b)$, is written as:
\begin{equation}
    s(a, b) = \sum_{u \in \mathcal{N}^{\text{CN}}_{(a, b)}} \mathbf{h}_u,
\end{equation}
where $\mathbf{h}_u$ is the node representation encoded by a MPNN.

\vskip 1em
\noindent {\bf NCNC}~\cite{ncn}: NCNC extends NCN by further considering the 1-hop neighbors of the node pair that aren't CNs. To account for the difference, they are weighted by the probability of they themselves being CNs of the other node in the pair. This is given for a target link $(a, b)$ as:
\begin{equation}
    s(a, b) = \sum_{u \in \mathcal{V}} w(a, b, u) \: \mathbf{h}_u,
\end{equation}
where 
\begin{equation}
        w(a, b, u) = \left\{\begin{array}{ll}
                            1, &\text{when } u \in \mathcal{N}^{\text{CN}}_{(a, b)} \\
                            \text{NCN}(A, X, b, u) &\text{when } u \in \mathcal{N}(a) \\
                            \text{NCN}(A, X, a, u) &\text{when } u \in \mathcal{N}(b) \\
                            0, &\text{else } 
                            \end{array}\right\}.
\end{equation}
This weighting scheme ensures that CNs play a larger role in the pairwise information than non-CNs.
\vskip 1em
\noindent {\bf BUDDY}~\cite{ncn}: BUDDY considers counting the number of nodes that correspond to different labels given by the double radius node labeling trick~\cite{labeling_trick}. We first define the number of nodes that are a distance $d_a$ and $d_b$ from nodes $a$ and $b$ as $\mathcal{A}_{ab}[d_a, d_b]$. We further define the number of nodes where $\text{max}(d_u, d_v) > k$ as $\beta_{ab}[d]$. The pairwise encoding concatenates the counts belonging to all combination of $d=1 \cdots k$. The counts are estimated using subgraph sketching algorithms~\cite{flajolet2007hyperloglog, broder1997resemblance} and are denoted $\hat{\mathcal{A}}$ and $\hat{\mathcal{B}}$. The pairwise encoding for a target link $(a, b)$ is given by the following where $[k] = \{1 \cdots k\}$:
\begin{align}
    s^{\hat{\mathcal{A}}}(a, b) &= \concat_{d_a, d_b \in [k]} \hat{\mathcal{A}}_{ab}[d_a, d_b], \\
     s^{\hat{\mathcal{\beta}}}(a, b) &= \concat_{d \in [k]} \hat{\beta}_{ab}[d], \\   
     s(a, b) &= s^{\hat{\mathcal{A}}}(a, b) \concat s^{\hat{\mathcal{\beta}}}(a, b).
\end{align}

\vskip 1em
\noindent {\bf Neo-GNN}~\cite{ncn}: Neo-GNN considers the higher-order neighbor overlap between two nodes. This is done by first learning a structural representation for each node $i$, $x_i^{struct}$. This is given by:
\begin{equation}
    x_i^{struct} = f_1 \left( \sum_{j \in \mathcal{N}(i)} f_2 \left(A_{ij} \right)\right).
\end{equation}
To consider the $L$-hop structural information, the structural representations are diffused over $L$ hops and weighted by a hyperparameter $\beta$: 
\begin{align}
    Z &= \text{MLP} \left( \sum_{l=1}^L \beta^{l-1} A^l X^{struct} \right), \\
    &\text{where} \:\:\: X=\text{diag}(x^{struct}).
\end{align}
The pairwise encoding $s(a, b)$ is the dot product of both the final representations, 
\begin{equation}
    s(a, b) = z_a^T z_b.
\end{equation}

\section{Special Cases of the General Pairwise Encoding} \label{sec:app_gen_formula}

In this section we demonstrate that multiple popular heuristics and pairwise encodings can be formulated as special cases of the general pairwise encoding given in Eq.~\eqref{eq:gen_pairwise}.
\vskip 1em
\noindent{\bf Common Neighbors (CNs)}~\cite{newman2001clustering}: The CNs of a pair of nodes $(a, b)$ is defined the overlapping 1-hop neighbors of both nodes:
\begin{equation}
    \mathcal{N}^{\text{CN}}_{(a, b)} = \mathcal{N}(a) \cap \mathcal{N}(b).
\end{equation}
Eq.~\eqref{eq:gen_pairwise} is equal to the CNs when $h(a, b, u)=1$ and $w(a, b, u)$ is:
\begin{equation} \label{eq:cn_weight2}
    w(a, b, u) = \left\{\begin{array}{ll}
                            1, &\text{when } u \in \mathcal{N}(a) \cap \mathcal{N}(b) \\
                            0, &\text{else } 
                            \end{array}\right\}.
\end{equation}

\noindent{\bf Adamic-Adar (AA)}~\cite{adamic2003friends}: AA is defined as the reciprocal log-degree weighted CN score where $d_u$ is the degree of node $u$:
\begin{equation}
    \text{AA}(a, b) = \sum_{u \in \mathcal{N}^{\text{CN}}_{(a, b)}} \frac{1}{\text{log}(d_u)}.
\end{equation}
Eq.~\eqref{eq:gen_pairwise} can be rewritten as the AA when $h(a, b, u)=1/\text{log}(d_u)$ and $w(a, b, u)$ is equal to Eq.~\eqref{eq:cn_weight2}.
\vskip 1em
\noindent{\bf Resource Allocation (RA)}~\cite{zhou2009predicting}: RA is defined as the reciprocal degree weighted CN score:
\begin{equation}
    \text{RA}(a, b) = \sum_{u \in \mathcal{N}^{\text{CN}}_{(a, b)}} \frac{1}{d_u}.
\end{equation}
Eq.~\eqref{eq:gen_pairwise} can be rewritten as the AA when $h(a, b, u)=1/d_u$ and $w(a, b, u)$ is equal to Eq.~\eqref{eq:cn_weight2}.
\vskip 1em
\noindent{\bf Katz Index}~\cite{katz1953new}: The Katz index is a global structural measure. It is defined as weighted summation of the number of paths of different lengths connecting $a$ and $b$. It is given by the following where the decay weight $\beta \in [0, 1]$,
\begin{equation}
    \text{Katz}(a, b) = \sum_{l=1}^{\infty} \beta^l A_{a, b}^l.
\end{equation}
This is equivalent to Eq.~\eqref{eq:gen_pairwise} when: 
\begin{equation}
    w(a, b, u) = \sum_{l=1}^{\infty} \beta^l e_a^T A^l,
\end{equation}
where $e_i \in \mathbb{B}^{\lvert \mathcal{V} \rvert}$ is a one-hot vector for a node $i$. We further set,
\begin{equation} \label{eq:b_equals_u}
     h(a, b, u) = \left\{\begin{array}{ll}
                            e_b^T, &\text{when } u=b \\
                            \mathbf{0}, &\text{else } 
                            \end{array}\right\}.
\end{equation}
\vskip 1em
\noindent{\bf Personalized Pagerank (PPR) Score}~\cite{pagerank}: The personalized pagerank score is the pagerank score localized to a root node $u$. The localization is via a teleportation probability $\alpha$ that transports the random walk back to the root node. We show that Eq.~\eqref{eq:gen_pairwise} can be rewritten as the PPR score when setting $h(a, b, u)$ equal to~\eqref{eq:b_equals_u} and, following \citet{chung2007heat}, setting $w(a, b, u)$ to:
\begin{equation}
    w(a, b, u) = \alpha \sum_{l=0}^{\infty} (1 - \alpha)^l e_a^T (D^{-1} A)^l.
\end{equation}

\noindent{\bf Feature Similarity}: The feature similarity of the pair of nodes $(a, b)$ is expressed by $\text{dis}(\mathbf{x}_a, \mathbf{x}_b)$ where $\mathbf{x}_a$ are the node features of node $a$ and $\text{dis}(\cdot)$ is a distance function (e.g., euclidean distance). This can be rewritten as Eq.~\eqref{eq:gen_pairwise} by substituting:
\begin{equation}
    w(a, b, u) = \text{dis}(\mathbf{x}_a, \mathbf{x}_u),
\end{equation}
and $h(a, b, u) = e_b^T$ where $e_i \in \mathbb{B}^{\lvert \mathcal{V} \rvert}$ is a one-hot vector for a node $i$.
\vskip 1em
\noindent{\bf NCN}~\cite{ncn}: The pairwise encoding used in NCN is defined as the summation of the representations for the CNs of a link. Eq.~\eqref{eq:gen_pairwise} can be rewritten as NCN when $w(a, b, u)$ is equal to Eq.~\eqref{eq:cn_weight2}. $h(a, b, u)$ is equal to the node representation $u$ encoded by a MPNN, i.e., $h(a, b, u) = \mathbf{h}_u$ where $H = \text{MPNN}(A, X)$.
\vskip 1em
\noindent{\bf NCNC}~\cite{ncn}: NCNC extends NCNC by further weighting the 1-hop (non-CN) by their probability of linking to the other nodes. Given Eq.~\eqref{eq:gen_pairwise},  the weight $w(a, b, u)$ is equal to following where 1-hop neighbors are weighted by their probability of linking with the other node:
\begin{equation}
    w(a, b, u) = \left\{\begin{array}{ll}
                            1, &\text{when } u \in \mathcal{N}^{\text{CN}}_{(a, b)} \\
                            \text{NCN}(A, X, b, u) &\text{when } u \in \mathcal{N}(a) \\
                            \text{NCN}(A, X, a, u) &\text{when } u \in \mathcal{N}(b) \\
                            0, &\text{else } 
                            \end{array}\right\}.
\end{equation}
$\text{NCN}(A, X, a, u)$ is the probability of $a$ and $u$ being linked using the NCN model. We further define $h(a, b, u) = \mathbf{h}_u$.

\vskip 1em

\noindent {\bf Neo-GNN}~\cite{yun2021neo}: The pairwise encoding used in Neo-GNN considers the higher-order neighborhood overlap between two nodes. The formulation is given in Section~\ref{sec:app_gen_formula}. When $l=1$, it can be expressed using Eq.~\eqref{eq:gen_pairwise} by setting:
\begin{equation}
    h(a, b, u) = f_1 \left( \sum_{v \in \mathcal{N}(u)} f_2 \left(A_{uv} \right)\right)^2,
\end{equation}
and 
\begin{equation}
    w(a, b, u) =  \left\{\begin{array}{ll}
                            1, &\text{when } u \in \mathcal{N}^{\text{CN}}_{(a, b)} \\
                            0, &\text{else } 
                            \end{array}\right\}.
\end{equation}

\section{Proof of Proposition~\ref{th:ppr}} \label{sec:proof1}

\fta*

\begin{proof}
    Per~\citet{chung2007heat}, the PPR vector for a root node $s$, $\text{pr}_{s}$, is equivalent to:
    \begin{equation} \label{eq:ppr_diffusion}
        \text{pr}_{s} = \alpha \sum_{k=0}^{\infty} (1-\alpha)^k W^k x_s ,
    \end{equation}
    where $W$ is a the random walk matrix and $x_s$ is a preference vector that is a one-hot vector for element $s$. We note that $\text{pr}_{s}(t)$ represents the landing probability of node $t$ given the root node $s$. As such, by definition, $\text{pr}_{s}(t) = \text{ppr}(s, t)$. Furthermore, it is clear that $r_s^k = W^k x_s \in \mathbb{R}^{\mathcal{V}}$ represents the probability of a walk of length $k$ beginning at node $s$ and stop all other nodes, individually. Also, the probabilities of all walks of length $k$ are weighted by $\gamma^k = \alpha (1-\alpha)^{k}$.
    $\Gamma \left(a, b, u\right)$ can be obtained by first taking the sum of the PPR vectors for nodes $a$ and $b$,
    \begin{align} \label{eq:ppr_mean}
         \text{pr}_{a} + \text{pr}_{b} &=  \alpha \sum_{k=0}^{\infty} (1-\alpha)^k W^k x_a + \alpha \sum_{k=0}^{\infty} (1-\alpha)^k W^k x_b, \nonumber \\
         \text{pr}_{a, b}    &=  \alpha \sum_{k=0}^{\infty} (1-\alpha)^k W^k \left( x_a + x_b \right) ,
    \end{align}
    where $\text{pr}_{a, b} = \text{pr}_{a} + \text{pr}_{b}$. 
    From this, we can express $\Gamma(a, b, u)$ as:
    \begin{align}
        \Gamma(a, b, u) &= \text{ppr}(a, u) + \text{ppr}(b, u), \nonumber \\
                        &= \text{pr}_{a, b}(u), \\
                        &= \text{pr}_{a}(u) + \text{pr}_{b}(u), \nonumber
    \end{align}
    which as shown in Eq.~\eqref{eq:ppr_mean} is equivalent to the probability of a walk that originates from either node $a$ or $b$ and terminates at node $u$.  This completes the proof.
\end{proof}

\section{Attention Formulation} \label{sec:app_attention}

For a target link $(a, b)$, LPFormer attends to the nodes in the set $\bar{\mathcal{V}}(a, b)$. The attention mechanism used in LPFormer is defined in Section~\ref{sec:framework} as follows where $w(a, b, u)$ is the attention weight of $u$ to the target link and $\bar{\mathcal{V}}(a, b) = \mathcal{V} \setminus\{a, b\}$:
\begin{align} \label{eq:app_att}
    &\tilde{w}(a, b, u) = \phi \left(\mathbf{h}_a, \mathbf{h}_b, \mathbf{h}_u, \: \mathbf{rpe}_{(a, b, u)} \right), \nonumber \\
    &w(a, b, u) = \frac{\text{exp}(\tilde{w}(a, b, u))}{\sum_{v \in \bar{\mathcal{V}}(a, b)}\text{exp}(\tilde{w}(a, b, u))}.
\end{align}
The function $\phi(\cdot)$ is modeled via the attention mechanism defined in GATv2~\cite{gatv2}. We define $a \in \mathbb{R}^{2d'}$ and $W \in \mathbb{R}^{d \times d'}$. The raw attention weights are then given by:
\begin{align}
    \tilde{w}(a, b, u) = \mathbf{a}^T \: \text{LeakyReLU} \left[W \:\mathbf{h}_a \concat W \: \mathbf{h}_b \concat W \: \mathbf{h}_u \concat \mathbf{rpe}_{(a, b, u)} \right].
\end{align}
The final attention weights, ${w}(a, b, u)$, are given by passing $\tilde{w}(a, b, u)$ through a softmax activation layer.

\section{Additional Experimental Details} \label{sec:app_exp}


\subsection{Planetoid splits} \label{sec:planetoid_splits}

We note that for each of Cora, Citeseer, Pubmed we use a fixed split. This follows the recent work of~\cite{li2023evaluating}. \citet{li2023evaluating} observe that for Cora, Citeseer, Pubmed there exists no unified data split between studies. They find that while recent work~\cite{chamberlain2022graph, ncn} use 10 random splits, prior work~\cite{nbfnet, velickovic2017graph} use a fixed split and train over 10 random seeds. Furthermore, there exists discrepancies in the preprocessing between those works that use the random splits. \citet{chamberlain2022graph} only use the largest connected component of each dataset while \citet{ncn} use the whole dataset. This makes any comparison of the published results difficult. Due to these discrepancies, we use the performance on the fixed split given by~\citet{li2023evaluating}, as {\it it's the only split where all methods are evaluated and compared under the same setting}.  

\subsection{Omission of ogbl-ddi under the Existing Evaluation} \label{sec:ddi}

We further omit the results of ogbl-ddi in Table~\ref{tab:main_results}. This is due to the observation made by \citet{li2023evaluating} that there exists a poor relationship between the validation and test performance. This extends to recent pairwise MPNNs, including NCN~\cite{ncn}, Neo-GNN~\cite{yun2021neo}, and BUDDY~\cite{chamberlain2022graph}. This makes tuning on the validation set difficult, as it doesn't guarantee good test performance. Due to this, they observe that when tuning on a fixed set of hyperparameter ranges, they are unable to achieve comparable results to the reported performance. Often they observe that the performance is actually much lower. Due to these concerns we believe ogbl-ddi is not suitable for the task of transductive link prediction and don't report the performance. For more details and discussion, please see Appendix D in \citet{li2023evaluating}. However, they show that this problem does not afflict ogbl-ddi under the newly proposed HeaRT~\cite{li2023evaluating} evaluation setting. As such, we further include the results for our method under HeaRT in Table~\ref{tab:heart_full}.

\subsection{Computation of the PPR Matrix} \label{sec:app_ppr}

We compute the PPR matrix via the efficient approximation algorithm introduced by~\citet{andersen2006local}. The estimation is controlled by a tolerance parameter $\epsilon$. The parameter $\epsilon$ controls both the speed of computation and the sparsity of the solution (i.e., a higher value of $\epsilon$ will produce a sparser PPR matrix). We use: $\epsilon=1e^{-7}$ for Cora and Citeseer, $\epsilon=5e^{-5}$ for ogbl-collab and ogbl-ppa, $\epsilon=1e-5$ for Pubmed, and $\epsilon=5e^{-3}$ for ogbl-Citation2. The value of $\epsilon$ is chosen as a trade-off between accuracy and sparsity to allow for ease of storage in GPU memory.

\subsection{Splitting Target Links by LP Factor} \label{sec:app_factor_details}

In Section~\ref{sec:exp_factor} we demonstrate the performance on samples that correspond to a single LP factor. In this section we further detail the algorithm used to determine the set of samples corresponding to each factor. We consider the three main factors: local structural information, global structural information, and feature proximity. We measure each using a single representative heuristic: CNs~\cite{newman2001clustering} for local information, PPR~\cite{pagerank} for global information, and cosine feature similarity for feature proximity. For each sample, we check if the score is only high in {\bf one} heuristic. In this way, it tells us that there is a dominant factor present in the pairwise information.

This determination is done by comparing the the heuristic scores of each target link against a threshold value. For a LP factor $i$ and target link $(a, b)$, we denote the heuristic score as $s^{i}(a, b)$. The threshold value for factor $i$ is represented by $\hat{s}^i$ and is chosen such that it corresponds to a higher score. We desire $\hat{s}^i$ to be a higher score such that any score $\geq$ than it indicates that a plethora of pairwise information exists corresponding to factor $i$. This is done by setting the threshold equal to the $p$-th percentile value for that heuristic among all target links. For example, for CNs, the 80th percentile score on one dataset may be 9. The value of $p$ is chosen to be high (e.g., 80\%) due to the aforementioned reasoning. Given these inputs, for each target link we compare the score for factor $i$ against the threshold value of that factor. Continuing our example, if $(a, b)$ only has 2 CNs, it is below the previously defined threshold. We only consider a sample as ``belonging'' to a single factor when it is $s^i(a, b) \geq \hat{s}^i$ is true {\bf for one only one factor $i$}. So if the heuristic score for $(a, b)$ is below the $p$-th percentile threshold for CNs and PPR but above for feature similarity, then feature proximity will be considered the dominant LP factor. However, if it's above the threshold for both local and structural information, it will not be assigned to any group. This is done as we want to isolate links that only highly express one LP factor. This allows us to better understand how certain methods can model that specific factor. The detailed algorithm is given in Algorithm~\ref{alg:lp_factor}.

We note that each target link may not belong to a category. This can be due to there being no or many dominant LP factor.  We further set the percentile equal to 90\% on all datasets except for ogbl-collab for which we use 80\%. These values were chosen as we wanted the percentile to be suitably high such that we are confident that the corresponding factor is relevant to the target link. Furthermore, we use a lower value for ogbl-collab as we found it produced a more even distribution of links by factor.

\begin{figure*}[t]
\centering
    \begin{subfigure}{.5\textwidth}
      \centering
      \includegraphics[width=0.8\linewidth]{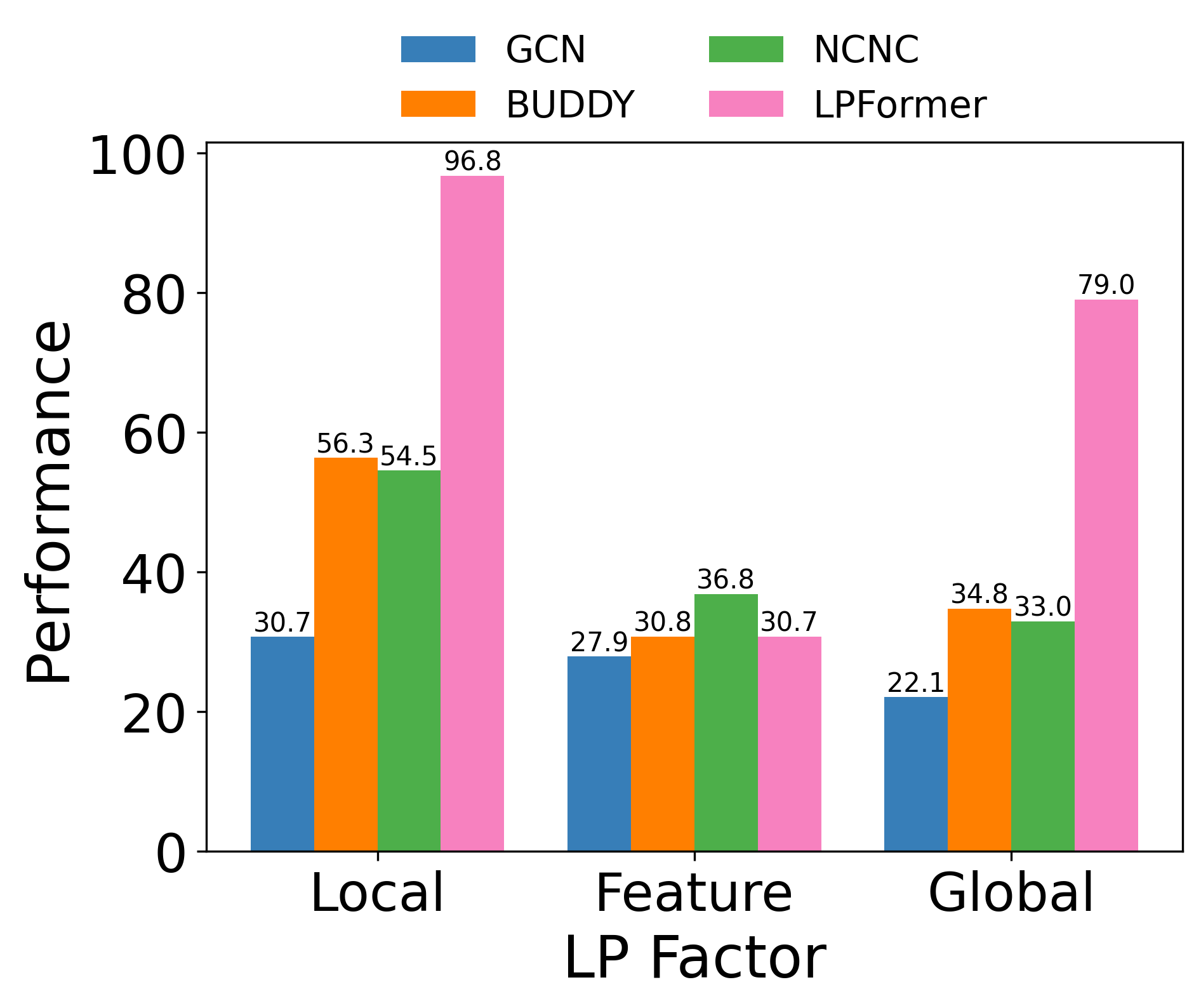}
      \caption{Pubmed}
      \label{fig:factor_pubmed}
    \end{subfigure}%
    \begin{subfigure}{.5\textwidth}
      \centering
      \includegraphics[width=0.8\linewidth]{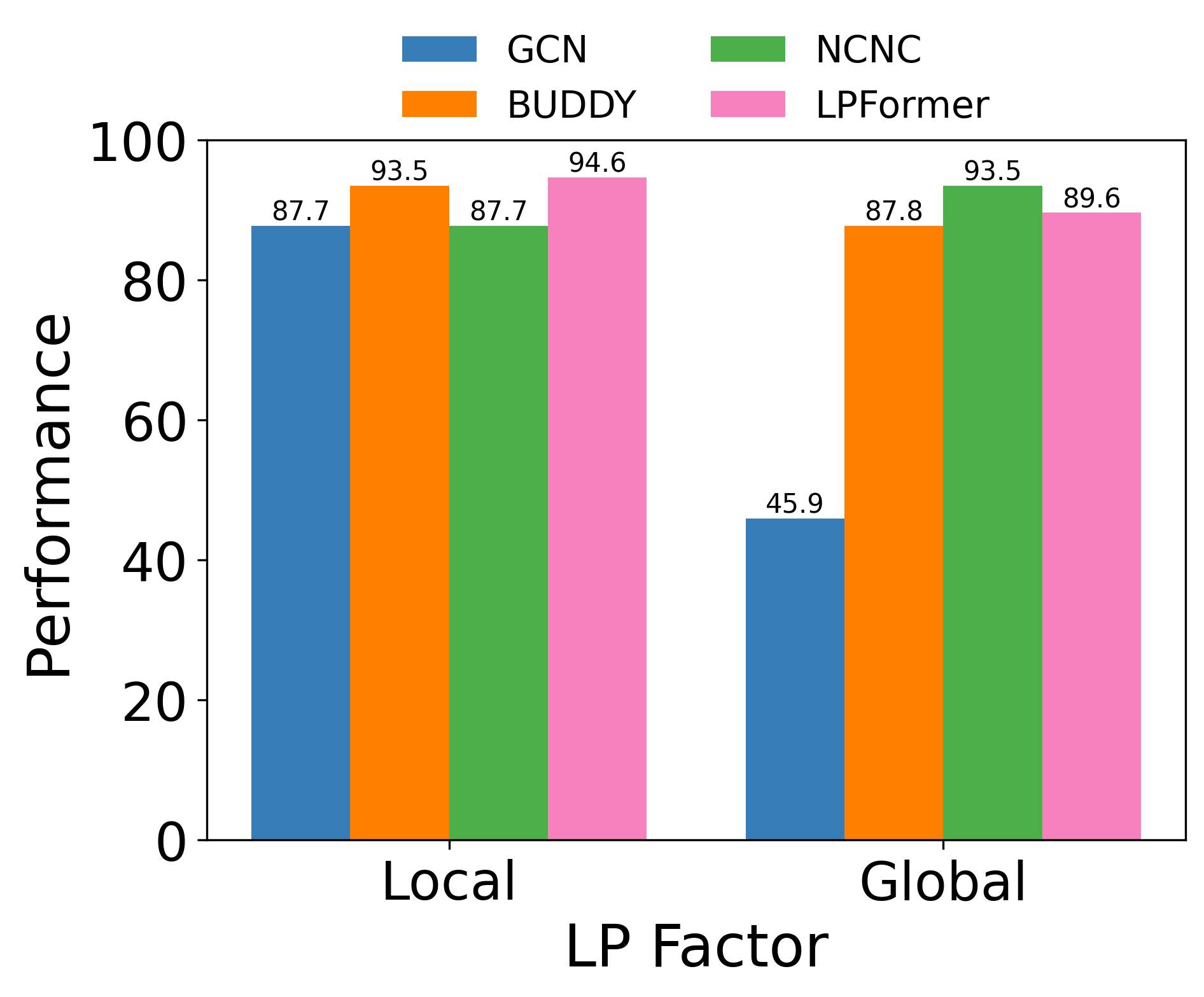}
      \caption{ogbl-ppa}
      \label{fig:factor_ppa}
    \end{subfigure}%
    \caption{Performance for target links when there is only one LP factor strongly expressed. Results are on (a) Pubmed, (b) ogbl-ppa. We note that due the quality of features used, we omit the feature proximity factor for ogbl-ppa from our analysis}

\label{fig:exp_factors_app}
\end{figure*}

\subsection{Additional Results for the LP Factor Experiments} \label{sec:app_factors_exp}

In Section~\ref{sec:exp_factor} we observed the performance of various methods on target links where only a single LP factor is expressed. This is done through the use of heuristic scores. We further demonstrate the results on the Pubmed and ogbl-ppa datasets. Of note is that for ogbl-ppa the initial node features are one-hot vectors that signify the species that the protein belongs to. We observe that due to the sparseness of these features, feature proximity measures are unable to properly predict any target links on their own. As such, the factor corresponding to feature proximity is not expressed. We therefore exclude that factor for this analysis on ogbl-ppa.

The results for both Pubmed and ogbl-ppa datasets are given in Figure~\ref{fig:exp_factors_app}. As shown earlier in Figure~\ref{fig:exp_factors}, LPFormer can most consistently perform well across each factor. This suggests that LPFormer is best able to both model a variety of factors and adapt accordingly for each target link.

\subsection{Performance on Heterophilic Datasets} \label{sec:heterophilic}

In this section we evaluate LPFormer on multiple heterophilic datasets. Heterophily refers to the tendency of dissimilar nodes to be connected. This is as opposed to homophily, in which nodes with similar attributed are more likely to be connected. Since most graphs used for benchmark datasets tend to contain homophilic patterns, heterophilic graphs present an interesting challenge regarding the effectiveness of graph-based methods. For a more detailed discussion on heterophilic graphs, please see~\cite{mao2024demystifying}.

We test on two prominent heterophilic datasets, Squirrel and Chameleon~\cite{heterophilic_datasets}. The statistics for each are in Table~\ref{table:heterophilic_stats}. We limit our comparison to those LP methods that tend achieve the best results, including GCN, BUDDY, and NCNC. In Table~\ref{tab:heterophilic_results}, we report the MRR over five random seeds. Note that we test under the {\bf original evaluation setting} and not HeaRT. We observe that LPFormer can achieve a large increase over other methods, with a 14\% and 9\% increase in performance on Squirrel and Chameleon, respectively. These results indicate the superior ability of LPFormer to accurately model LP on heterophilic graphs, as compared to other methods.

\begin{table}[h]
\centering
\vskip -0.25em
 \caption{Heterophilic Dataset Statistics. 
 }
 \vskip -1em
    \begin{tabular}{ccc}
    \toprule
     & Squirrel & Chameleon \\
     \midrule
    \#Nodes &  5201 & 2277 \\
    \#Edges & 198,353 & 31,371 \\
    Split Ratio & {85/5/10} & {85/5/10} \\
    \bottomrule
    \end{tabular}
    \label{table:heterophilic_stats}
\end{table}

\begin{table}[H]
\centering
    \vskip -1em
 \caption{Results on Heterophilic Datasets. 
 }
 \vskip -1em
    \begin{tabular}{c|cc}
    \toprule
     {\bf Method} & {\bf Squirrel} & {\bf Chameleon} \\
     \midrule
    {\bf GCN} & 22.77 ± 4.54 & 20.74 ± 8.08 \\ 
    {\bf BUDDY} & 9.69 ± 0.99 & 6.30 ± 2.40 \\ 
    {\bf NCNC} & 32.37 ± 5.46 & 26.24 ± 3.37 \\ 
    \midrule 
    {\bf LPFormer} & {\bf 36.77 ± 2.77} & {\bf 28.61 ± 6.68} \\ 
    \% Improvement & 14\% & 9\% \\
    \bottomrule
    \end{tabular}
    \label{tab:heterophilic_results}
\end{table}

\subsection{More Efficiently Incorporating the PPR Scores} \label{sec:app_ppr_new}

In Figure~\ref{fig:runtime} we compare the training time between LPFormer and NCNC. We observe that on the denser datasets, ogbl-ppa and ogbl-ddi, LPFormer is considerably more efficient. Furthermore, on ogbl-collab, both methods have a fast runtime. However, we find that LPFormer struggles on ogbl-citation2 in comparison to NCNC. We observe that this is due to the need of the PPR matrix, which while sparse, requires a large amount of memory and processing time. In the future, we plan to fix this problem by performing a simple and efficient pre-processing step. Specifically, before training, we can iterate over all target links and extract the relevant PPR scores. This would obviate the need to store the PPR matrix and determine the nodes for each link. Furthermore, this only needs to be done once before tuning the model. This would greatly reduce the storage and time needed to train LPFormer on all datasets and is an avenue we plan to explore in the future.

\begin{algorithm}[h]
\small
\caption{Determining Samples by LP Factor} \label{alg:lp_factor}
\begin{algorithmic}[1]

\Require
    \Statex $\text{CN}(\cdot)$ = Maps $(i, j)$ to \# of CNs of the pair
    \Statex $\text{PPR}(\cdot)$ = Maps $(i, j)$ to PPR score of the pair
    \Statex $\text{FS}(\cdot)$ = Maps $(i, j)$ to feature cosine similarity of the pair
    \Statex $p$ = Percentile used to determine whether a factor is present
    \Statex $\mathcal{E}^{\text{test}}$ = Positive test links
\newline
\State // Compute the score corresponding to the $p$-th percentile for each heuristic
\State $\hat{s}^{\text{CN}} = \text{Percentile}(p, \{CN(i, j) \: | \: (i, j) \in \mathcal{E}^{\text{test}} \})$
\State $\hat{s}^{\text{FS}} = \text{Percentile}(p, \{FS(i, j) \: | \: (i, j) \in \mathcal{E}^{\text{test}} \})$
\State $\hat{s}^{\text{PPR}} = \text{Percentile}(p, \{PPR(i, j) \: | \: (i, j) \in \mathcal{E}^{\text{test}} \})$
\newline 
\State Create empty lists $L^{\text{CN}}$, $L^{\text{PPR}}$, and $L^{\text{FS}}$
\For{$(i, j) \in \mathcal{E}^{\text{test}}$}
    \State link-cn  = $\text{CN}(i, j)$
    \State link-fs  = $\text{FS}(i, j)$
    \State link-ppr = $\text{PPR}(i, j)$
    \newline
    \State // Assign sample to corresponding list based on scores
    \If{$\text{link-cn} \geq \hat{s}^{\text{CN}}$ {\bf and} $\text{link-fs} < \hat{s}^{\text{FS}}$ {\bf and} $\text{link-ppr} < \hat{s}^{\text{PPR}}$ } 
        \State Append($L^{\text{CN}}$, $(i, j)$)
    \ElsIf{$\text{link-cn} < \hat{s}^{\text{CN}}$ {\bf and} $\text{link-fs} \geq \hat{s}^{\text{FS}}$ {\bf and} $\text{link-ppr} < \hat{s}^{\text{PPR}}$ } 
        \State Append($L^{\text{FS}}$, $(i, j)$)
    \ElsIf{$\text{link-cn} < \hat{s}^{\text{CN}}$ {\bf and} $\text{link-fs} < \hat{s}^{\text{FS}}$ {\bf and} $\text{link-ppr} \geq \hat{s}^{\text{PPR}}$ } 
        \State Append($L^{\text{PPR}}$, $(i, j)$)    
    \EndIf
\EndFor
\State \Return $L^{\text{CN}}$, $L^{\text{PPR}}$, $L^{\text{FS}}$

\end{algorithmic}
\end{algorithm}

\end{document}